\documentclass[12pt]{article}
\usepackage{graphicx} 
\usepackage{xcolor}  
\usepackage{amsfonts}
\usepackage{bbm}
\usepackage{bm}
\usepackage{mathtools}
\usepackage{subcaption}
\usepackage{amsmath}
\usepackage[colorlinks,urlcolor=blue,linkcolor=blue,citecolor=blue]{hyperref}
\usepackage{booktabs} 
\usepackage{amsthm}
\usepackage{amssymb}
\usepackage{array}
\usepackage{algorithm}
\usepackage{algpseudocode} 
\usepackage[margin=1in]{geometry} 
\usepackage{authblk} 
\newtheorem{proposition}{Proposition}
\theoremstyle{definition} 
\newtheorem{definition}{Definition}

\DeclareMathOperator*{\argmax}{arg\,max}
\DeclareMathOperator*{\argmin}{arg\,min}

\usepackage{natbib}
\title{Goal Exploration via Adaptive Skill Distribution for Goal-Conditioned Reinforcement Learning}
\author[1]{Lisheng Wu\thanks{Email: \href{mailto:lisheng.wu@manchester.ac.uk}{lisheng.wu@manchester.ac.uk}}}
\author[1]{Ke Chen\thanks{Email: \href{mailto:ke.chen@manchester.ac.uk}{ke.chen@manchester.ac.uk}}}
\affil[1]{Department of Computer Science, University of Manchester, Manchester, M13 9PL, UK}
\date{}

\begin{document}

\maketitle

\begin{abstract}
Exploration efficiency poses a significant challenge in goal-conditioned reinforcement learning (GCRL) tasks, particularly those with long horizons and sparse rewards. A primary limitation to exploration efficiency is the agent's inability to leverage environmental structural patterns. In this study, we introduce a novel framework, GEASD, designed to capture these patterns through an adaptive skill distribution during the learning process. This distribution optimizes the local entropy of achieved goals within a contextual horizon, enhancing goal-spreading behaviors and facilitating deep exploration in states containing familiar structural patterns. Our experiments reveal marked improvements in exploration efficiency using the adaptive skill distribution compared to a uniform skill distribution. Additionally, the learned skill distribution demonstrates robust generalization capabilities, achieving substantial exploration progress in unseen tasks containing similar local structures.
\end{abstract}

\section{Introduction}
Goal-conditioned reinforcement learning (GCRL) equips agents with the capability to tackle a multitude of tasks, each delineated by distinct goals. Within the GCRL paradigm, environments often present sparse-reward configurations stemming from the non-trivial intricacies inherent in reward engineering. Such settings pose a significant challenge to efficient exploration, particularly in the pursuit of desired goals within long-horizon tasks.

In response to the challenges of efficient exploration in GCRL, several strategies have emerged. These include setting goals that maximize the entropy of achieved outcomes \citep{pong2019skew,pitis2020maximum}, employing world models for the discovery of novel goals \citep{mendonca2021discovering,hu2023planning}, and orchestrating exploration through intermediate sub-goals \citep{hoang2021successor}. However, their effectiveness in acquiring new goals hinges upon random action noise and specialized target policies, such as those conditioned on selected sub-goals \citep{pong2019skew,pitis2020maximum}. Notably, these strategies tend to confine agents within restrictive exploratory areas, sacrificing the opportunity to explore more new goals.

Instead of relying on primitive actions for exploration, skills \citep{florensa2017stochastic,eysenbach2018diversity,campos2020explore,gehring2021hierarchical} encompass a range of effective behaviors designed to target specific objectives systematically over multiple steps, enabling the agent to explore beyond restrictive areas  \citep{wu2024goal,xu2022aspire,hao2023skill,pertsch2021accelerating}. However, the efficacy of each skill is typically confined to a specific subset of the state space and their application outside these subsets often leads to less impactful and even negligible changes in the environment. To evaluate the efficacy of skills, the Value Function Space (VFS) \citep{shah2021value} employs skill value functions collectively as a generalizable representation. Concurrently, Skill Prior Reinforcement Learning (SPiRL) focuses on learning the distribution of skill priors to effectively prioritize among skills in each scenario. Nonetheless, both approaches encounter adaptability challenges in unseen states and often overlook the implications of skill selections on the exploration objective, a crucial aspect of deep exploration \citep{osband2016deep,osband2019deep}.

In addressing the limitations of existing exploration methods, we introduce Goal Exploration via Adaptive Skill Distribution (GEASD), a novel approach that fosters unsupervised adaptation of skill distribution and in-depth exploration. The objective of our skill distribution is to optimize the lower bound of the global exploration objective by focusing on the local entropy of achieved goals within a historical contextual horizon. As shown in Fig. \ref{fig:intuitive_vis}, we aim to leverage the structural information of the historical context to guide evolution into future scenarios. Motivated by VFS \citep{shah2021value}, GEASD features a structural representation based on skill value functions. This representation is dynamically refined through intrinsic rewards that quantify local entropy variations, effectively capturing the evolving dynamics of historical contexts. Following this, we adopt an adaptive skill distribution in the form of a Boltzmann distribution, derived directly from the skill value functions. This skill distribution accounts for the local entropy in future scenarios, thereby facilitating in-depth exploration.

\begin{figure}
    \centering
    \includegraphics[width=\textwidth]{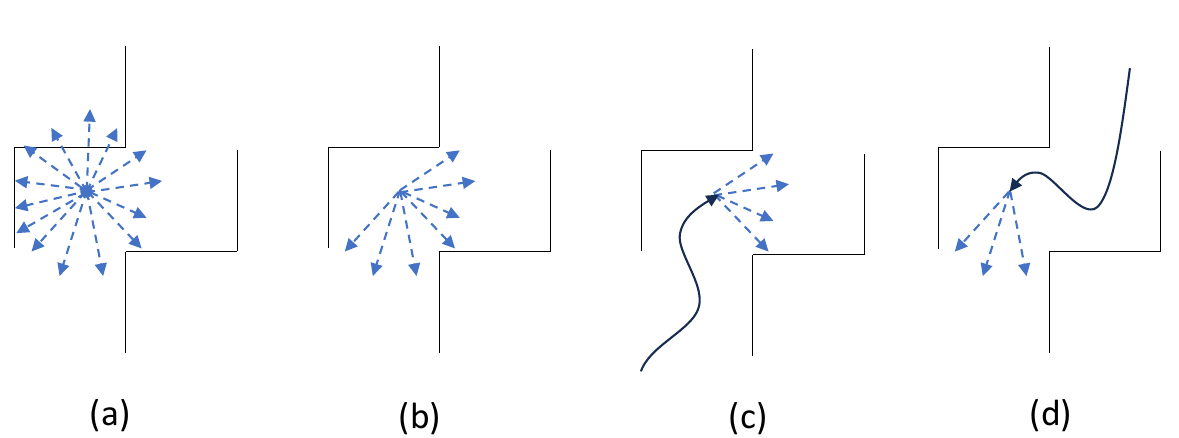}
    \caption{Visualizing decision-making in a maze grid: Dashed arrows indicate potential future actions, while solid arrows trace the agent's historical trajectory. In (a), the agent considers all potential actions equally within its action space, not yet incorporating information from the environmental structure. In (b), it strategically exploits the local environmental structure to selectively focus on actions that would broaden the spread of achieved goals. In (c), the agent's previously established trajectory makes actions leading downwards appear less favorable. In (d), the established leftward trajectory implies that actions moving to the right may be disadvantageous.}
    \label{fig:intuitive_vis}
\end{figure}

In summary, our work makes the following key contributions:
i) We introduce an innovative framework GEASD within a non-Markovian exploration process.
ii) We develop novel intrinsic rewards to quantify local entropy changes and an adaptive skill distribution, underpinned by skill value functions, to facilitate deeper exploration. iii) We provide a theoretical justification for the entropy optimization of our skill distribution, contingent upon certain skill-related assumptions.
iv) Through comparative benchmarks, we demonstrate the superior deep exploration efficiency and generalization capabilities exhibited by our approach relative to state-of-the-art methods.

\section{Related Work}
\paragraph{Exploration via Intrinsic Rewards} In RL, intrinsic rewards are pivotal for guiding agents through environments characterized by sparse rewards and extensive horizons, often linked to the novelty of states or dynamics. This novelty is broadly classified into two categories: inter-episode and intra-episode novelty, each emphasizing the agent's discovery of new goals across and within episodes, respectively.
Inter-episode novelty is commonly associated with significant prediction errors \citep{stadie2015incentivizing,pathak2017curiosity,kim2018emi,bai2021variational}, low state visitation frequencies \citep{tang2017exploration,bellemare2016unifying,ostrovski2017count,machado2020count,choshen2018dora}, and substantial model disagreements \citep{burda2018exploration,pathak2019self,shyam2019model}. These intrinsic rewards direct the agent towards novel states, enhancing the likelihood of encountering previously unseen states.
To foster intra-episode novelty, some strategies calculate intrinsic rewards based on the distance from states previously visited within the same episode \citep{savinov2018episodic,badia2020never}. Meanwhile, other approaches ensure these rewards are issued only once for each distinct state or area, aiming to minimize repetitive exploration \citep{stanton2018deep,zhang2021noveld}. Employing intrinsic rewards based on intra-episode novelty helps circumvent the problem of frequently revisiting identical states during exploration, thereby encouraging a wider range of exploratory behaviors.
However, methods from both categories primarily rely on random actions, the introduction of noise into target policies, or a combination of both to uncover new states, which is inefficient. This strategy often results in exploration being concentrated around already known states that are deemed novel due to their high intrinsic rewards, consequently limiting the discovery of unseen states and resulting in shallow exploration
Our approach to intrinsic rewards, while also aimed at enhancing intra-episode novelty, diverges in its ultimate goal. We aim to develop an exploration policy that efficiently navigates novel scenarios with a structure familiar to the agent. This ambition moves us beyond the conventional scope of simply optimizing known scenarios with after-the-fact intrinsic rewards, enabling the utilization of structural information to perform deep exploration.

\paragraph{Exploration in GCRL}
In GCRL, agents begin without a clear understanding of how to reach their goals, making efficient exploration strategies essential for discovering and achieving these targets. Agents enhance their exploration by targeting various sub-goals in addition to the desired goal. Moreover, achieving a sub-goal provides the agent with intrinsic rewards, shifting the focus from rewarding novel states to setting strategic sub-goals.
Several methods have been developed to facilitate this exploration process: Some create goals of intermediate difficulty to avoid overly simple or complex tasks \citep{florensa2018automatic,campero2020learning}, others sample goals from a distribution uniformly spread across previously achieved goals to promote balanced exploration \citep{pong2019skew}, and yet others prioritize exploration of goals with the lowest density or visit counts \citep{pitis2020maximum,hoang2021successor}.
To extend the exploration beyond the neighboring areas along the trajectories towards the sub-goals, some works extensively perform goal-independent exploration after reaching these sub-goals \citep{pitis2020maximum,hoang2021successor}.
Alternatively, other works establish specific goal-independent policies for exploration and incorporate world models to enhance sample efficiency through training with model rollouts \citep{mendonca2021discovering,hu2023planning}.
However, despite their innovative approaches, these methods encounter challenges similar to those faced by intrinsic reward-based methods, particularly in their ability to effectively explore unseen goals. Although PEG \citep{hu2023planning} utilizes world model rollouts to identify sub-goals with elevated discovery potential, the model's imprecision in capturing the dynamics of novel areas restricts the agent's ability to explore extensively.
To overcome these limitations, GEAPS \citep{wu2024goal} integrates pre-trained skills \citep{florensa2017stochastic,eysenbach2018diversity} as effective behavioral patterns to augment exploration. This integration prevents agents from stagnating in specific states, thereby improving the efficiency of discovering new states. However, the efficacy of these skills may be limited to certain contexts, and GEAPS's uniform skill distribution approach might reduce their effectiveness in enhancing exploration efficiency. In contrast, our method introduces a learnable skill distribution, specifically designed to facilitate more efficient navigation through the environment by the agent.

\paragraph{Exploration via Skills}
Skills, as an extended form of action, are often adopted in hierarchical reinforcement learning (HRL). As the efficacy of skills is often constrained to specific states, more efficient exploration demands the capability to evaluate their efficacies. VFS \citep{shah2021value} can abstract the state via skill value functions, which enables the agent to directly evaluate the efficacy of each skill in each scenario. A learned policy based on the VFS representation can generalize to novel scenarios with similar representations, even facilitating zero-shot generalization. SPiRL \citep{pertsch2021accelerating} and ASPiRe \citep{xu2022aspire} acquire skill distribution priors conditioned on states from pre-training datasets, which can indicate the relative efficacy of skills in each state. However, these methods require pre-training tasks or offline datasets that mirror the structures of the current task to learn accurate skill value functions or establish effective skill priors. Otherwise, they can even adversely affect exploration efficiency for those unfamiliar states, especially in sparse-reward settings where they almost have no chance to correct the behaviors before the achievement of the desired goal. Moreover, they do not account for the impacts of skills on the exploration objective, thereby limiting the depth of exploration \citep{osband2016deep, osband2019deep}. For example, some skills may lead to the revisitation of previously explored states, rendering the previous efforts almost in vain. In contrast, our skill distribution supports unsupervised adaptation to the current task and accounts for the impacts of skill on the local entropy within a historical context to promote deep exploration. In our framework, we utilize skill value functions as structural representations, akin to the VFS approach. However, we distinctively leverage these skill values to capture expected local entropy changes and directly influence the skill distribution, optimizing exploration without necessitating further policy learning. Furthermore, our integration of skill value functions with skill distribution significantly reduces the risk of catastrophic failures often associated with policy relearning during adaptation. When applying skills within GCRL, we utilize skills in a coarse-to-fine manner, as exemplified by GEAPS \citep{wu2024goal}. This method not only facilitates the effective discovery of novel goals but also supports the learning of fine-grained, goal-conditioned policies from low-level data.

\section{Preliminary}
GCRL is framed as a goal-augmented Markov Decision Process (GAMDP) characterized by parameters 
$(\mathcal{S}, \mathcal{A}, \mathcal{T}, r, \mathcal{G}, p_{d\!g}, \phi, \gamma, K)$. Here, $\mathcal{S}$, $\mathcal{A}$, $\gamma$, and $K$ are state space, action space, discount factor, and horizon, respectively. The transition function \(\mathcal{T}(s'|s,a)\) gives the probability of transitioning from state \(s\) to \(s'\) via action \(a\). The goal space is $\mathcal{G}$, with $p_{d\!g}$ as the desired goal distribution. The function $\phi(s)$ maps state $s$ to its achieved goal.
In sparse reward settings, the reward \(r(s', g)\) is binary, being zero if the distance \(d(\phi(s'), g)\) between achieved and target goals is less than \(\epsilon\) and \(-1\) otherwise. GCRL's objective is to learn a goal-conditioned policy $\pi(s, g)$ maximizing:
\begin{equation}
    J(\pi) = \mathbb{E}_{g\sim p_{d\!g}, a_t \sim \pi(s_t, g), s_{t+1}\sim \mathcal{T}(\cdot|s_t, a_t)}\left[\sum_{t=0}^{K-1} \gamma^t r(s_{t+1}, g)\right].  \nonumber
\end{equation}
We represent the state at timestep \(t\) as \(s_t\) and define the historical trajectory up to this point as \(h_t = (s_0, a_0, \ldots, s_{t-1}, a_{t-1}, s_t)\), which encompasses the sequence of states and actions. A historical context \(h^C_t\), spanning \(t-C\) to \(t\) for context horizon \(C\), is represented by \(h^C_t = (s_{t-C+1}, a_{t-C+1}, \ldots, s_{t-1}, a_{t-1}, s_{t})\). For \(C > K\), \(h^C_t\) is equivalent to \(h_t\). Additionally, we introduce the mapping function \(\bm{\Phi}\), which transforms \(h^C_t\) into a sequence of achieved goa2ls, denoted by \(\bm{\Phi}(h^C_t) = (\phi(s_{t-C+1}), \phi(s_{t-C+2}), \ldots, \phi(s_t))\).

\section{Methodology}

\subsection{Overview}
Entropy maximization is a widely adopted strategy for exploration in RL. Existing methods, while effective at guiding agents to explore novel scenarios through intrinsic rewards, typically fall short of extending exploration to unseen areas in-depth. This limitation largely stems from their reliance on random action noises or specific target policies, which lack awareness of the structural information in novel states. Our approach seeks to overcome this by harnessing local structural information, guiding the agent to undertake deep explorations into these previously unexplored areas, thus significantly expanding the exploration frontier.

To achieve this target, our exploration policy concentrates on optimizing the local entropy of achieved goals within a specified historical context. This strategic optimization enhances the lower bound of the overall entropy of achieved goals. To optimize local entropy effectively, the exploration policy would learn to utilize structural information from the historical context, generating actions that result in a broader diversity of achieved goals. Moreover, when faced with novel scenarios that present structural information familiar to the agent, it can still devise actions that extend achieved goals into unexplored areas, even in the absence of precise dynamics. All technical issues stated above are addressed in Section~\ref{sect:exploration_obj}.

Central to our method is \textit{Skill-based Local Entropy-Maximization Pattern} (SLEMP), an ideal behavior distribution of a predefined set of skills that maximizes the local entropy with those skills. SLEMP adapts as the historical context evolves, promoting a flexible approach to exploring various scenarios. All technical issues stated above are resolved in Section~\ref{sect:sturctural_pattern}.

To accurately capture and utilize the structural information from the historical context, we leverage skill value functions as our structural representation. We establish novel intrinsic rewards to reflect the local entropy changes during the historical context's evolution and learn the skill value functions based on those intrinsic rewards. Therefore, this structural representation can reflect their relative capabilities in optimizing local entropy, which can be seen as their capabilities to spread out the achieved goal in the current scenario. The details of our novel structural representation is described in Section~\ref{sect:structural_representation}.

Building upon these skill value functions, we propose employing a skill distribution modeled after the Boltzmann distribution. We adopt a dynamic temperature for the agent to flexibly prioritize deeper exploration with the skill of highest value when the local entropy is high, while seeking more options when the local entropy is low. The resulting adaptive skill distribution can assist our agents in navigating through the environment more efficiently. The technical details on the adaptive skill distribution can be found in Section~\ref{sect:skill_dist_derivation}.

Lastly, we introduce the \textit{General Exploration via Adaptive Skill Distribution} (GEASD) framework as presented in Section~\ref{sect:geasd}. GEASD is designed to integrate seamlessly with existing GCRL exploration methodologies, coupled with our innovative skill distribution approach. This integration enables a deep and efficient exploration policy to explore the desired goals more efficiently.

\subsection{Exploration Objective}
\label{sect:exploration_obj}
We formulate the objective for the exploration policy to be the local entropy of achieved goals as $H(\bm{\Phi}(h^C_t))$. 
The overall entropy of all achieved goals until $t$th step of the current episode takes all experiences into account, including those collected in previous episodes, is denoted as $H_{\text{all}}(\mathcal{G})$. To analyze the influence of the local entropy, we divide the overall experiences until the $t$th step into two parts at the $(t-C)$th step, resulting in the entropy decomposition into the entropy of all achieved goals until $(t-C)$ step of the current episode, denoted as $H_{\text{past}}(\mathcal{G})$ and the local entropy $H(\bm{\Phi}(h^C_t))$. As demonstrated by \cite{wu2024goal}, a lower bound of the overall entropy $H_{\text{all}}(\mathcal{G})$ can be expressed as 
\begin{equation}
    H_{\text{all}}(\mathcal{G}) \geq \beta H_{\text{past}}(\mathcal{G}) + (1-\beta) H(\bm{\Phi}(h^C_{t})), \label{eq:entropy_ineq}
\end{equation}
where \( \beta = \frac{C}{|\mathcal{B}|} \) 
represents the ratio between the contextual horizon and the cumulative number of steps historically stored in the replay buffer (\( |\mathcal{B}| \)). The derivation of this inequality is based on Jensen's inequality \citep{wu2024goal} and the equality holds only when any achieved goals within the historical context $\bm{\Phi}(h^C_t)$ do not overlaps with all achieved goals prior to the $(t-C)$th step. 
By the lower bound as shown in Eq.~(\ref{eq:entropy_ineq}), optimizing our local entropy objective $H(\bm{\Phi}(h^C_t))$ can contribute to the enhancement of the overall exploration objective, which promotes the expansion of achieved goals to discover the desired goals. 

In practice, the agent cannot control the experiences that have happened and needs to make actions optimizing $H(\bm{\Phi}(h^C_t))$ before the $t$th step. In other words, the agent needs to account for the local entropy within the historical context in the future. In our work, we specifically aim to optimize  $H(\bm{\Phi}(h^C_{t+k})|h^C_t)$ $(0<k<C)$, which denotes the local entropy in $k$ step given the information contained in $h_t$. In the progressive evolution of historical contexts, properly enhancing the local entropy of the most recent context helps us achieve our goal of maximizing local entropy. During the learning, we particularly aim to leverage the model to make efficient exploration actions based on its identification of the structural information in scenarios that are even novel to the agent. 

It is worth noting that optimizing $H(\bm{\Phi}(h^C_{t+k})|h^C_t)$ is not the same as choosing a single future segment $h^{k}_{t+k}$ that leads to the maximum value. Instead, $h^{k}_{t+k}$ follows a distribution jointly decided by the transition dynamics and exploration policies. Therefore, $\bm{\Phi}(h^C_{t+k})$ represents a distribution of achieved goals with deterministic historical achieved goals $\bm{\Phi}(h^{C-k}_t)$ and stochastic future achieved goals $\bm{\Phi}(h^k_{t+k})$.

Although the study by \cite{wu2024goal} on GEAPS also focuses on the local entropy, they optimize the local entropy via \(H(\bm{\Phi}(h^k_{t+k}) | \psi(s_t))\), wherein \(\psi(s_t)\) predominantly captures the agent's internal state to aid in generalization. This approach neglects the broader context, excluding important structural details from both the historical context. In contrast, our formulation, \( H(\bm{\Phi}(h^C_{t+k})|h^C_t) \), enhances the optimization of \( H(\bm{\Phi}(h^C_t))\) by incorporating a more comprehensive set of structural information from \(h^C_t\), thereby providing a more holistic approach to understanding and interacting with the environment.

\subsection{Structural Patterns}
\label{sect:sturctural_pattern}
Consider the example shown in Fig.~\ref{fig:intuitive_vis}d: From the perspective of local entropy maximization, it tends to move in a manner that disperses the achieved goals. Simultaneously, it aims to avoid actions leading to traveling backward along the historical path or colliding with walls, which would prevent the agent from traveling further. To accomplish this, the agent must utilize structural information, including the local layout of the maze it perceives and the historical trajectories. This structural information is sufficient to indicate the potential relative changes in each behavior concerning the local entropy $H(\bm{\Phi}(h^C_t))$. We define these local entropy changes as:
\begin{definition}
\label{def:local_entropy_change}
\textbf{(Local Entropy Change)} Local entropy change quantifies the variation in local entropy within a system constrained by a sequence of states over a fixed context horizon $C$ as time $t$ progresses. The local entropy change for the $k$-step transition from $h^C_t$ to $h^C_{t+k}$ is given by: 

\begin{equation}    
H\left(\bm{\Phi}\left(h^C_{t+k}\right)|h^C_{t+k}\right) - H\left(\bm{\Phi}\left(h^C_t\right)|h^C_t\right) \nonumber
\end{equation}
\end{definition}

Recall our optimization objective in Section \ref{sect:exploration_obj}, where we aim to find the distribution of action sequences spanning $k$ steps that leads to the optimization of $H(\bm{\Phi}(h^C_{t+k})|h^C_t)$. However, considering all possible action sequences, their number grows exponentially as $|\mathcal{A}|^k$ with the planning horizon $k$, making it nearly impossible to exhaustively examine all behaviors for tasks with large action spaces and long planning horizons.
Furthermore, most behaviors have minimal impact on the environment, even without obstructing environmental structures. For instance, actions resulting in jittering around the original position do not significantly alter the environmental state. In contrast, effective behaviors, typically a small subset of all behaviors, have the potential to traverse more diverse achieved goals.
Therefore, learning the distribution of those effective behaviors can reduce computational complexity while still largely maintaining the effectiveness of optimizing our objective $H(\bm{\Phi}(h^C_{t+k})|h^C_t)$. It is worth noting that the skills are executed using a \textit{coarse-to-fine} strategy, as described in GEAPS \citep{wu2024goal}, which enables the exploration of a greater variety of goals for the goal-conditioned policy to learn from.

To capture effective behaviors in achieving distinct goals, skill learning \citep{florensa2017stochastic,eysenbach2018diversity,wu2024goal,pertsch2021accelerating} has emerged as a promising approach. In this framework, each skill aims to encapsulate a mode of effective behavior, and these skills can be either pre-trained in pre-training environments or on offline datasets, presented in either categorical or continuous form. Our discussion focuses on categorical skills. It is worth clarifying how to obtain the optimal set of skills in a general form is beyond the scope of this research.
In our work, we characterize each skill by a latent vector $\bm{z}$, which is used to sample actions according to $a_t \sim \sigma(a_t | \psi(s_t), \bm{z})$ at state $s_t$, where $\psi(s_t)$ is a feature mapping of the state $s_t$. We denote the set of skills as $\mathcal{Z}$ and set the skill horizon to be the same as the planning horizon $k$.
Furthermore, we define the distribution of skills conditioned on the historical context $h^C_t$ as $D(\bm{z}|h^C_t)$. 
Based on our objective $H(\bm{\Phi}(h^C_{t+k})|h^C_t)$, we provide the following definition of a structural pattern of $h^C_t$ from a behavioral perspective:
\begin{definition}
(\textbf{Skill-based Local Entropy-Maximization Pattern}) A Skill-based Local Entropy-Maximization Pattern (\textit{SLEMP}) for the scenario $h^C_t$ is defined using a distribution $D_{max}(\bm{z}|h^C_t)$ over a set of skills $\mathcal{Z}$ with skill horizon $k$. This distribution maximizes the local entropy $H(\bm{\Phi}(h^C_{t+k})|h^C_t)$ and is derived in detail through the following steps:
\begin{gather}
    p(h^k_{t+k}|s_t, \bm{z}) = \prod_{i=0}^{k-1} \mathcal{T}(s_{t+i+1}|s_{t+i}, a_{t+i})\sigma(a_{t+i}|\psi(s_{t+i}), \bm{z}), \nonumber \\ 
p_{D}(g | h^C_t) = \frac{\mathbb{E}_{z \sim D(\cdot | h^C_t), h^k_{t+k} \sim p(\cdot | s_t, \bm{z})} \left[ \left| \{ g' \in \bm{\Phi}(h^{C-k}_t \oplus h^k_{t+k}) : g' = g \} \right| \right]}{C} , \label{eq:skill_goal_cover} \\
    D_{max}(\bm{z}|h^C_t) = \argmax_{D(\bm{z}|h^C_t)} H_D(G|h^C_t) = \argmax_{D(\bm{z}|h^C_t)} \mathbb{E}_{g \sim p_D(g|h^C_t)} [-\log p_D(g|h^C_t)]. \nonumber
\end{gather}
where the symbol $\oplus$ in Eq.~(\ref{eq:skill_goal_cover}) denotes the concatenation of history segments, merging $h^{C-k}_t$ with $h^k_{t+k}$ into a single history segment.
\end{definition}
The SLEMP, by its nature as a pattern, can inherently recur across diverse scenarios. Therefore, an accurate prediction of SLEMPs in unseen scenarios can still assist the agent in navigating those scenarios efficiently.
For effective derivation of SLEMP, it is crucial that the agent comprehends the structural information embedded within the historical context $h^C_t$.  This necessitates a meaningful structural representation, which will be the focus of the subsequent section.

\subsection{Structural Representation via Skill Value Functions} \label{sect:structural_representation}
To comprehend the essential structural information for deducing SLEMP, the capacity to assess local entropy changes following skills is indispensable. Motivated by this understanding, we advocate for the employment of expected local entropy changes from all skills collectively to construct a meaningful structural representation of this structural information. Such expected local entropy changes can be expressed as 
\begin{equation}
    \Delta H(h^C_t, \bm{z}) = \mathbb{E}_{h^k_{t+k}\sim p(\cdot|s_t, \bm{z})}\left[H(\bm{\Phi}(h^C_{t+k})|h^C_{t+k}) - H(\bm{\Phi}(h^C_t)|h^C_t)\right] \label{eq:expected_local_entropy_changes}
\end{equation}
\begin{figure}
    \centering
    \includegraphics[width=\textwidth]{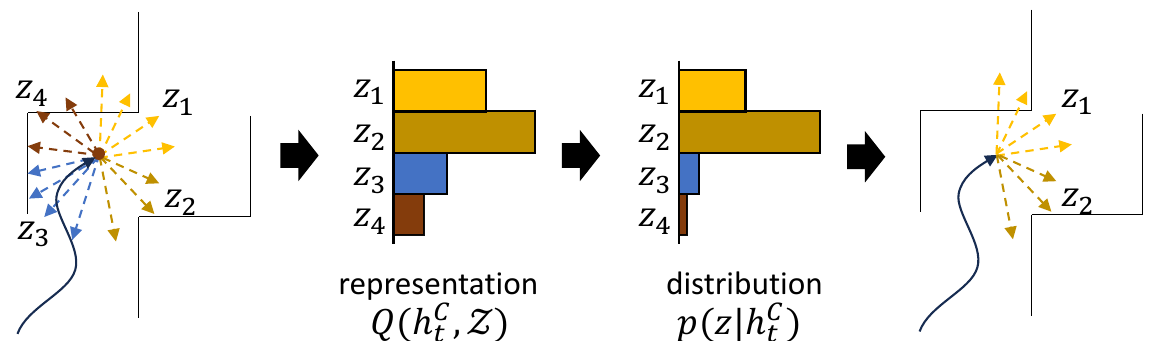}
    \caption{Illustration of Our Learning Pipeline: The leftmost figure visualizes structural information, including historical paths and local layouts, as indicated by \(h^C_t\). It also showcases potential trajectories for the four skills characterized by \(\{\bm{z}_i\}^4_{i=1}\). Following the leftmost figure, we present the structural representation \(Q(h^C_t, \mathcal{Z})\), derived from the aggregated value functions of skills. Subsequently, we derive a Boltzmann distribution \(p(\bm{z}|h^C_t)\) based on these representations. Consequently, skills characterized by lower value functions are assigned minimal probabilities, leading to their marginalization, as depicted in the rightmost figure. }
    \label{fig:learning_process}
\end{figure}

Drawing upon the VFS framework \citep{shah2021value}, which employs skill value functions collectively as state representations, we integrate this concept into our methodology. Specifically, we adapt the approach of evaluating expected local entropy changes in the form of skill value functions.  To facilitate the learning of such skill value functions, we introduce novel intrinsic rewards for the history context transition $(h^C_t, h^C_{t+1})$ as:
\begin{equation}
r_{\text{info}}(h^C_t, h^C_{t+1}) = H(\bm{\Phi}(h^C_{t+1})) - H(\bm{\Phi}(h^C_t)) \nonumber
\end{equation}
This reward function is specifically designed to encourage the achievement of distinct goals that yield high local entropy gains within the context horizon $C$. Conventionally, the high-level target value of the skill value function for the skill-level historical context transition $(h^C_t, \bm{z}_t, h^C_{t+k})$ is defined as follows:
\begin{equation}
y_{\text{high}}(h^C_t, \bm{z}_t, h^C_{t+k}|\xi(s_t)=1) =
\sum_{i=0}^{k-1} r_{\text{info}}(h^C_{t+i}, h^C_{t+i+1}),\label{eq:hierarchical_target}
\end{equation}
where $\xi(s_t) = 1$ signals the commencement of the current skill, characterized by $\bm{z}_t$, from state $s_t$. If $\xi(s_t) \neq 1$, the executed skill might vary across the subsequent skill horizon, rendering the data inappropriate for learning the skill value function for $\bm{z}_t$. $z^{\text{sample}}_{t+k}$ is sampled from the learned conditional skill distribution $p(\bm{z}|h^C_{t+k})$. However, the available data to derive the target values for each scenario are limited to historical context $h^C_t$ with $\xi(s_t)=1$, which can lead to sample inefficiency in learning skill functions.

To improve the sample efficiency, we propose a low-level version of the target value for a low-level context transition $(h^C_t, a_t, h^C_{t+1})$ as
\begin{equation}
    y_{\text{low}}(h^C_t, \bm{z}, h^C_{t+1}) = r_{\text{info}}(h^C_t, h^C_{t+1}) + \hat{\gamma} Q(h^C_{t+1}, \bm{z}). \label{eq:flat_target}
\end{equation}
where $\hat{\gamma}$ is a discount factor intentionally set to be smaller than $\gamma$, which is used in the acquisition of goal-conditioned behaviors, to prioritize short-term returns. In our work, $\gamma$ is set to $(1-\frac{1}{k})$ so that $\hat{\gamma}^{k-1} > e^{-1} \approx 0.3679$, as it is a monotonically decreasing function with the following property:
\begin{equation}
\lim_{k \to \infty} \left(1-\frac{1}{k}\right)^k = e^{-1}, \nonumber
\end{equation}
which ensures that the agent focuses on short-term returns without being overly shortsighted, allocating at least an $e^{-1}$ weight to the reward in $k$ steps. In Eq.~\eqref{eq:flat_target}, $\bm{z}$ is sampled from $p(\bm{z}|s_t, a_t)$ during the learning process as:
\begin{equation}
    p(\bm{z}|s_t, a_t) = \frac{\sigma(a_t|\psi(s_t), \bm{z})}{\sum_{\bm{z}'}\sigma(a_t|\psi(s_t), \bm{z}')}. \nonumber
\end{equation}
Different from $y_{\text{high}}$, $y_{\text{low}}$ is directly bootstrapped from the skill value function of the same skill $\bm{z}$ at $(t+1)$th step without resampling. This approach allows for learning the skill value function at a granular level without defined skill boundaries. However, this approach could lead to potential oversights in long-term strategic planning, especially when the skill horizon \(k\) is notably shorter than the scope of structural information available to the agent. This mismatch might prevent the agent from considering various skill combinations that are within the perceived scope.

For the high-level target value function, as delineated in Equation~(\ref{eq:hierarchical_target}), the skill value function loss is formulated as:
\begin{equation}
  L_{\text{high}} = \mathbb{E}_{(s_t, h^C_t, \bm{z}_t, h^C_{t+k}) \sim \mathcal{B}}\left [ \left( y_{\text{high}}(h^C_t, \bm{z}_t, h^C_{t+k}) - Q(h^C_t, \bm{z}_t) \right)^2| \xi(s_t) = 1 \right], \nonumber
\end{equation}

Conversely, the loss function for the low-level target, as depicted in Equation~(\ref{eq:flat_target}), incorporates the importance sampling ratio, \(\rho_t\), into its formulation:
\begin{equation}
  L_{\text{low}} = \mathbb{E}_{(s_t, a_t, h^C_t, h^C_{t+1}) \sim \mathcal{B}, z \sim p(\bm{z}|s_t, a_t)}\left[\rho_t \cdot \left( y_{\text{low}}(h^C_t, \bm{z}, h^C_{t+1}) - Q(h^C_t, \bm{z}) \right)^2 \right], \nonumber
\end{equation}
where \(\rho_t = \frac{\sigma(a_t|\psi(s_t), \bm{z})}{\pi_b(a_t|s_t)}\) adjusts for the difference in action selection probabilities between the behavior \(\pi_b(a_t|s_t)\) and the target skill policy $\sigma(a_t|\psi(s_t), \bm{z}_t)$, thereby mitigating off-policy bias. This adjustment ensures a more accurate estimation of the skill value functions under the target skill policy, even with data from differing policies. 

Consequently, the learned skill value functions collectively highlight the relative benefits of different skills, quantified in terms of local entropy changes for each scenario, thereby providing a structural representation that is directly pertinent to exploration efforts.

\subsection{Skill Distribution Derivation} \label{sect:skill_dist_derivation}
With the learned structural representation consisting of skill-value functions \(Q(h^C_t, \bm{z})\), $z\sim\mathcal{Z}$, we construct the skill distribution in the form of Boltzmann distribution as:
\begin{equation}
    p(\bm{z}|h^C_t) = \frac{e^{-E(\bm{z})}}{\sum_{\bm{z}'}{e^{-E(\bm{z}')}}} = \frac{e^{Q(h^C_t, \bm{z})}}{\sum_{\bm{z}'} e^{Q(h^C_t, \bm{z}')}}, \nonumber
\end{equation}
where \(E(\bm{z}) = -Q(h^C_t, \bm{z})\) serves as the energy function, inversely correlating with the skill's value function. This ensures that skills with higher skill values are more likely to be selected, while also providing opportunities for the selection of lower-valued skills. The rationale of this form is provided in Section \ref{sect:theoretical_analysis}.

In practice, if the context horizon $C$ is much larger than the skill horizon $k$, the difference between skill value functions can be minor, which impedes the agent from exploiting the learned structural information to make a difference. Moreover, the skill values might not be accurate and may trap the agent in a local minima, requiring additional exploration. To strike a better balance between exploration and exploitation, we introduce a temperature parameter \(T_{\text{dyna}}\) to modulate the balance between exploration and exploitation. The value of \(T_{\text{dyna}}\) is determined by the local entropy \(h^C_t\) and ranges from \(T_{\text{min}}\) to \(1\), where \(0 < T_{\text{min}} < 1\). It is defined as:
\begin{equation}
    T_{\text{dyna}} = e^{\min\left(\frac{H(\bm{\Phi}(h^C_t))}{H^{C}_{\text{max}}}, 1\right)\cdot\log T_{\text{min}}}, \label{eq:dyna_temperature}
\end{equation}
where \(H^{C}_{\text{max}}\) is the maximum recorded local entropy in the historical window of length \(C\). A high local entropy \(H(\bm{\Phi}(h^C_t))\) signals effective skill selection, favoring a lower temperature \(T_{\text{dyna}}\) to further exploit the skill with the highest skill value. Conversely, a low local entropy \(H(\bm{\Phi}(h^C_t))\) often indicates an impasse or complex scenarios requiring careful reevaluation, prompting a higher temperature \(T_{\text{dyna}}\) to promote exploration among different skills. Therefore, \(T_{\text{dyna}}\) responsively reflects the agent's immediate exploratory context, adjusting the exploration strategy in real-time. The adaptive skill distribution with $T_{\text{dyna}}$ can be formulated as:
\begin{equation}
    p_{\text{dyna}}(\bm{z}|h^C_t) =  \frac{e^{Q(h^C_t, \bm{z})/T_{\text{dyna}}}}{\sum_{\bm{z}'} e^{Q(h^C_t, \bm{z}')/T_{\text{dyna}}}}.\label{eq:dyna_temperature_dist}
\end{equation}

As depicted in Fig.~\ref{fig:learning_process}, by directly mapping the skill-value function to the skill distribution function, our approach eliminates the need for an additional learning mechanism for the exploration strategy. This aspect notably distinguishes our work from the VFS approach by Shah et al. \citep{shah2021value}, which requires supplementary policy learning on the value function space. Our strategy not only simplifies the exploration process but also ensures efficient exploration and dynamic adaptation to the evolving understanding of the environment.

\subsection{Goal Exploration Strategy}
\label{sect:geasd}
In this section, we present a detailed framework aimed at optimizing skill distribution, named Goal Exploration via Adaptive Skill Distribution (GEASD). This framework systematically divides the exploration process into two key stages: navigation towards novel areas and exploration empowered by the adaptive skill distribution mechanism.

In regions that have been thoroughly explored, the extensive historical data stored in the replay buffer offers greater insights than the limited history context $h^C_t$, which is confined to the current episode only. Consequently, we implement a sub-goal selection strategy that directs the agent towards novel areas not yet extensively explored. In this context, we opt for the OMEGA sub-goal selection strategy \citep{pitis2020maximum}, specifically designed to favor achieved goals characterized by low density within the historical data. It is crucial to emphasize that this sub-goal selection strategy is flexible and can be replaced with any appropriate alternative. Subsequently, we employ our adaptive skill distribution strategy to enhance exploration, leveraging structural information for deeper exploration. The comprehensive algorithm is detailed in Algorithm~\ref{alg:geasd}. In the GEASD framework, we introduce GEASD-H and GEASD-L to distinguish our methodologies based on the approach to learning the skill-value functions from the target values $y_\text{high}$ and $y_\text{low}$, as detailed in Section~\ref{sect:structural_representation}.

\begin{algorithm}[t]
\caption{Goal Exploration via Adaptive Skill Distribution (GEASD)}
\label{alg:geasd}
\textbf{Given:} a skill policy $\sigma$, a goal-conditioned policy $\pi$, a sub-goal selection policy $\pi_g$, a skill distribution $p$, a skill horizon $k$, a replay buffer $\mathcal{B}$, an episode horizon $K$, a context horizon $C$.
\begin{algorithmic}[1] 
\Procedure{GEASD}{}
    \State Initialize $z \gets \varnothing$, $t \gets 0$, $\Delta t \gets 0$, $h^C_t \gets (s_0,)$, Navigation-Flag $\gets 1$
    \State Sample $s_0$ from an initial state distribution and a sub-goal $g \sim \pi_g(\mathcal{G})$
    \While{$t \leq K$} 
        \If {Navigation-Flag $= 1$}
            \State Sample an action $a_t \sim \pi(s_t, g)$
        \Else
            \State $\Delta t \gets \Delta t + 1$
            \State Sample an action $a_t \sim \sigma(\psi(s_t), \bm{z})$
        \EndIf
        \State Obtain the next state $s_{t+1} \sim \mathcal{T}(s_{t+1}|s_t, a_t)$
        \State Update the context $h^C_{t+1} \gets h^{C-1}_t \oplus (a_t, s_{t+1})$
        \State Increment the time step $t \gets t + 1$
        \If {$g = \phi(s_t)$ and Navigation-Flag $= 1$} 
            \State Set Navigation-Flag $\gets 0$, reset $\Delta t \gets 0$ \Comment{Enter the second stage}
        \EndIf
        \If {$\Delta t$ mod $k = 0$ and Navigation-Flag $= 0$} 
            \State Compute $T_{dyna}$ according to Eq.~(\ref{eq:dyna_temperature}) and  reset $\Delta t \gets 0$
            \State Draw a skill $\bm{z}$ from $p_{\text{dyna}}(\bm{z}|h^C_t)$ as defined in Eq.~(\ref{eq:dyna_temperature_dist})
        \EndIf
        \State Save the tuple $(s_t, a_t, s_{t+1}, \bm{z}, g)$ in the replay buffer $\mathcal{B}$.
    \EndWhile
\EndProcedure
\end{algorithmic}
\end{algorithm}

\section{Theoretical Analysis}
\label{sect:theoretical_analysis}
In this section, we first conduct a comprehensive analysis to illuminate the reasoning behind choosing the Boltzmann distribution as the form of our skill distribution, subject to certain conditions. Subsequently, we examine the relationship between the expected local entropy changes in Section~\ref{sect:structural_representation}, which we approximate using skill value functions, and the optimization objective proposed in Section~\ref{sect:geasd}. The proofs of those propositions can be found in Appendix \ref{appendix:proof_propositions}.

Our goal is to maximize the local entropy \(H(\bm{\Phi}(h^C_{t+k})|h^C_t)\) after executing skills drawn from the skill distribution \(p(\bm{z}|h^C_t)\). To precisely quantify the impact of skill execution, the local entropy change resulting from the execution of a specific skill characterized by \(\bm{z}\) is formulated as:
\begin{align}
  \Delta H(\bm{\Phi}(h^C_{t+k})|h^C_t, \bm{z}) &= H(\bm{\Phi}(h^C_{t+k})|h^C_t, \bm{z}) - H(\bm{\Phi}(h^C_t)|h^C_t, \bm{z})  \label{eq:entropy_gain_origin} \\
  &= H(\bm{\Phi}(h^C_{t+k})|h^C_t, \bm{z}) - H(\bm{\Phi}(h^C_t)|h^C_t) \label{eq:entropy_gain_simp}
\end{align}

The transition from Eq.~(\ref{eq:entropy_gain_origin}) to Eq.~(\ref{eq:entropy_gain_simp}) arises because the distribution of historically achieved goals, given \(h^C_t\), is independent of subsequent skill selection. 

\begin{proposition} \label{prop:boltzamann_distribution}
Consider a set of skills denoted by \( \mathcal{Z} \), where each skill \( \bm{z}_i \in \mathcal{Z} \) uniquely covers a portion of the achieved goals, thereby ensuring \( H(\mathcal{Z}|s_t, \bm{\Phi}(h^k_{t+k})) = 0 \). The optimal skill distribution, conditioned on the historical trajectory \( h^C_t \) and aiming to maximize the entropy \( H(\bm{\Phi}(h^C_{t+k})|h^C_t) \), conforms to a form of the Boltzmann distribution. This distribution is expressed as
\begin{equation}
p^*(\bm{z}|h^C_t) = \frac{e^{-E^*(\bm{z})}}{\sum_i e^{-E^*(\bm{z}_i)}} = \frac{e^{\Delta H(\bm{\Phi}(h^C_{t+k})|h^C_t, \bm{z})}}{\sum_i e^{\Delta 
 H(\bm{\Phi}(h^C_{t+k})|h^C_t, \bm{z}_i)}}, \label{eq:Boltzmann_distribution} \nonumber
\end{equation}
where the energy function \(E^*(\bm{z})\) associated with skill \( z \) is quantified by the resulting entropy gain after execution, \( - \Delta H(\bm{\Phi}(h^C_{t+k})|h^C_t, \bm{z}) \).
\end{proposition}

Proposition \ref{prop:boltzamann_distribution} demonstrates that the Boltzmann distribution adequately serves as the form of skill distribution to optimize our objective, obviating the necessity for further learning under conditions \(H(\mathcal{Z}|s_t, \bm{\Phi}(h^k_{t+k})) = 0\). This proposition establishes the groundwork for our derivation of skill distribution, as outlined in Section \ref{sect:skill_dist_derivation}.

\begin{proposition}\label{prop:expect_lowerbound}
The local entropy change, \(\Delta H(\bm{\Phi}(h^C_{t+k})|h^C_t, \bm{z})\), resulting from executing skill \(\bm{z}\), is lower bounded by the expected change in local entropy across each trajectory \(h^k_{t+k}\) sampled from \(p(h^k_{t+k}|s_t, \bm{z})\). This relationship is formally described by the following inequality:
\begin{equation}
    \Delta H(\bm{\Phi}(h^C_{t+k})|h^C_t, \bm{z}) \geq \mathbb{E}_{h^k_{t+k} \sim p(\cdot|s_t, \bm{z})} \left[ H(\bm{\Phi}(h^C_{t+k})|h^C_{t+k}) - H(\bm{\Phi}(h^C_t)|h^C_t) \right]. \nonumber
\end{equation}
\end{proposition}

Recall the expected local entropy changes as outlined in Eq.~(\ref{eq:expected_local_entropy_changes}), which the skill value functions aim to approximate. Proposition~\ref{prop:expect_lowerbound} shows that the skill value functions are approximating a lower bound of \(\Delta H(\bm{\Phi}(h^C_{t+k})|h^C_t, \bm{z})\), serving as a conservative yet informative evaluation.

\section{Experiments}
In this section, we assess the effectiveness of our GEASD framework, focusing on its success rate, sampling efficiency, and the entropy of achieved goals. Our evaluation covers two challenging sparse-reward, long-horizon GCRL tasks. Additionally, we explore and visualize the exploratory behaviors exhibited by different methodologies. Subsequently, we examine the performance of skill execution within the learned skill distribution, extending our analysis to two unseen tasks characterized by sparse rewards and long horizons. Finally, an ablation study is conducted to delve into the impact of various hyperparameters within the GEASD framework, providing insights into their influence on overall performance.

\subsection{Environments and Baselines}
\subsubsection{Environments}
As shown in Fig.~\ref{fig:train_tasks}, our experimental framework encompasses two environments, each characterized by sparse rewards and extended horizons: i) \texttt{PointMaze-Spiral}: a two-dimensional maze in which an agent must navigate through a 10x10 maze configured in a counterclockwise spiral. The task demands precise navigation from the maze's bottom-left, denoted with the blue circle, to its center, marked by a red cross. Its observation is an eight-dimensional vector consisting of its position in the maze, its position relevant to the current grid, and a boolean vector indicating the existence of surrounding walls relative to the current grid. In the process, the agent may bump into walls or navigate backward along the incoming path, wasting exploratory efforts. ii) \texttt{AntMaze-U} \citep{pitis2020maximum,trott2019keeping}: involves a robotic control task where a four-legged robot must traverse a lengthy U-shaped corridor to attain the desired goal position. The observation space is thirty-dimensional, incorporating the robot's sensor data and spatial coordinates. The robot can only move in a slow and jittery manner, making it challenging to explore new goals and learn goal-reaching behavior. Successfully reaching the stipulated goals in both \texttt{PointMaze-Spiral} and \texttt{AntMaze-U} is deemed a success.

\begin{figure}[t]
    \centering
    \begin{subfigure}{0.24\textwidth}
        \centering
        \includegraphics[width=\linewidth]{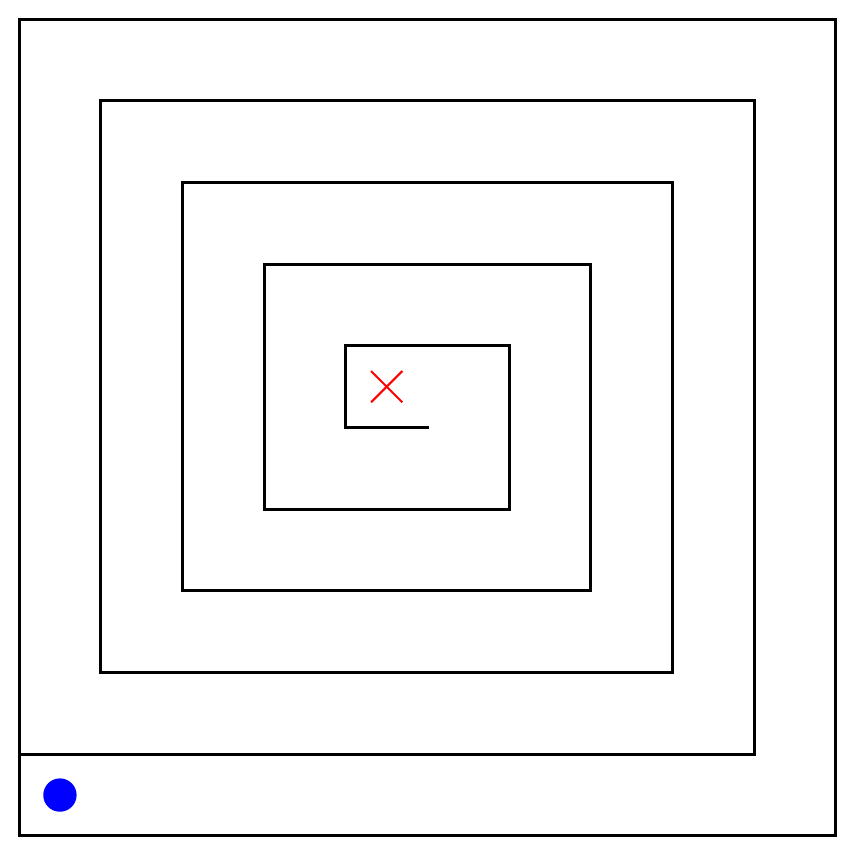}
        \caption{\texttt{PointMaze-Spiral}}
        \label{fig:point_maze_spiral_env}
    \end{subfigure}
    \hspace{3em}
    \begin{subfigure}{0.24\textwidth}
        \centering
        \includegraphics[width=\linewidth]{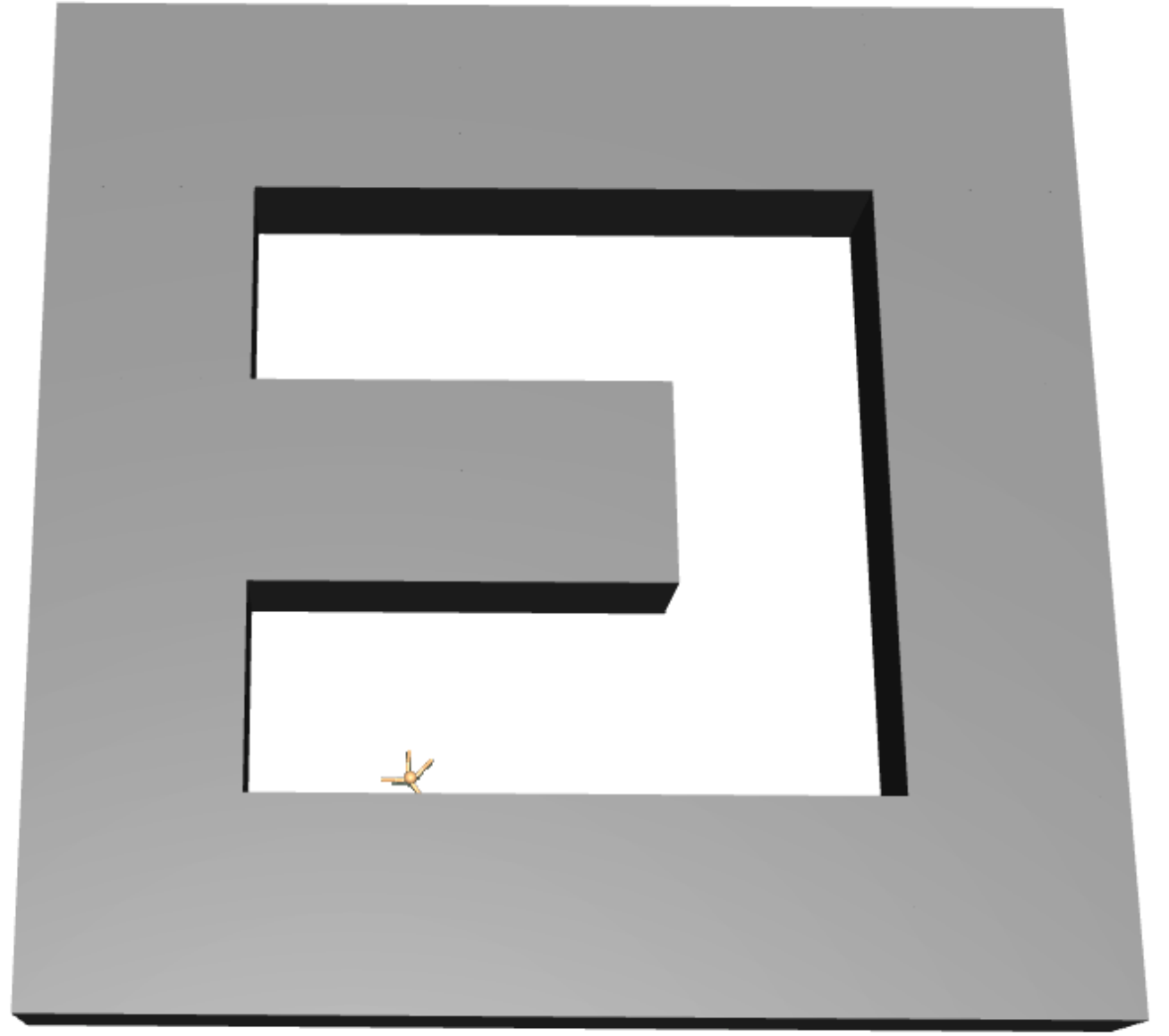}
        \caption{\texttt{AntMaze-U}}
        \label{fig:ant_maze_env}
    \end{subfigure}
    \caption{The experimental environments.}
    \label{fig:train_tasks}
\end{figure}

\subsubsection{Baselines}
Our study aims to assess the impact of skill-augmented exploration, particularly focusing on the effectiveness of adaptive skill distribution in enhancing exploration efficiency. Consequently, we benchmark our methodology against two foundational baselines: i) OMEGA \citep{pitis2020maximum}: This represents an advanced goal-exploration method that leverages primitive actions. Within OMEGA, sub-goals are strategically chosen based on the scarcity of achieved goals. ii) GEAPS \citep{wu2024goal}: This approach augments goal-exploration with pre-trained skills distributed uniformly. The strategy for selecting sub-goals in GEAPS mirrors that in OMEGA, emphasizing the role of skill-based augmentation in exploration tasks.

\subsection{Experiment Settings And Implementation}
\subsubsection{Experimental Settings} \label{sect:exp_settings}
This research aims to address several pivotal questions to understand our GEASD framework and our adaptive skill distribution. Specifically, we focus on the following inquiries:
Q1) How is learning efficiency improved through augmentation with skills via our adaptive skill distribution compared to uniform skill distribution?
Q2) What are the behavioral changes resulting from our adaptive distribution?
Q3) Can our learned skill distribution assist an agent in exploring novel environments with similar local structures?
Q4) What effects do the various hyper-parameters have on the learning performance?

For the comparative analysis, five trials with varied random seeds are conducted to ensure a robust evaluation of performance across each task. Assessments are bound by a predetermined budget, which dictates that training ceases after a specified number of steps if an agent has not met the ultimate goals. The evaluation criteria encompass the success rate, serving as a measure of task completion; sample efficiency, which gauges the efforts required to attain similar success rates; and the entropy of achieved goals, highlighting the overall exploratory objective.

\subsubsection{Implementation} \label{sect:implementation}
Our method, GEASD, is implemented utilizing the GEAPS codebase \citep{wu2024goal} and integrates pre-trained skills developed through the skill learning methodologies present in GEAPS. The learning of goal-conditioned policies, for both our approach and the comparative baselines, is facilitated by using the Deep Deterministic Policy Gradient (DDPG) algorithm \citep{lillicrap2015continuous}. For benchmarking purposes, we have employed the source codes of the baselines provided by the original authors, ensuring strict adherence to their documented guidelines.

The skill value functions within the GEASD framework are constructed using the Gated Recurrent Unit (GRU) \citep{cho-etal-2014-learning}, a type of recurrent neural network (RNN) designed to process sequential data by incorporating historical context into its input and producing value functions as output. Regarding the hyperparameters outlined in the methodology section, we adjust $k$ to $2$ and $25$, and $C$ to $10$ and $125$ for the \texttt{PointMaze-Spiral} and \texttt{AntMaze-U} tasks, respectively. Additionally, for the dynamic temperature calculation, $T_\text{min}$ is configured to $0.01$ for both tasks.

Appendix~\ref{sect:details} provides more technical and implementation details regarding our model and the baselines in our comparative study.

\subsection{Experimental Results}
In this section, we present the principal findings of our experimental investigation, aimed at addressing the five questions outlined in Section~\ref{sect:exp_settings}.

\subsubsection{Comparative Study} \label{sect:comp_study}
To address the first research question, we report the outcomes achieved by the baseline methodologies alongside our GEASD-L and GEASD-H implementations. In the conducted experiments, training was concluded after 200,000 steps for the \texttt{PointMaze-Spiral} task and one million steps for the \texttt{AntMaze-U} task. Each episode was defined as consisting of 150 steps for \texttt{PointMaze-Spiral} and 200 steps for \texttt{AntMaze-U}, respectively. The reported statistics, including both the mean and the standard deviation, were calculated over five different initializations (seeds) for each experimental environment.
\begin{figure}[t]
    \centering
    \includegraphics[width=0.8\textwidth]{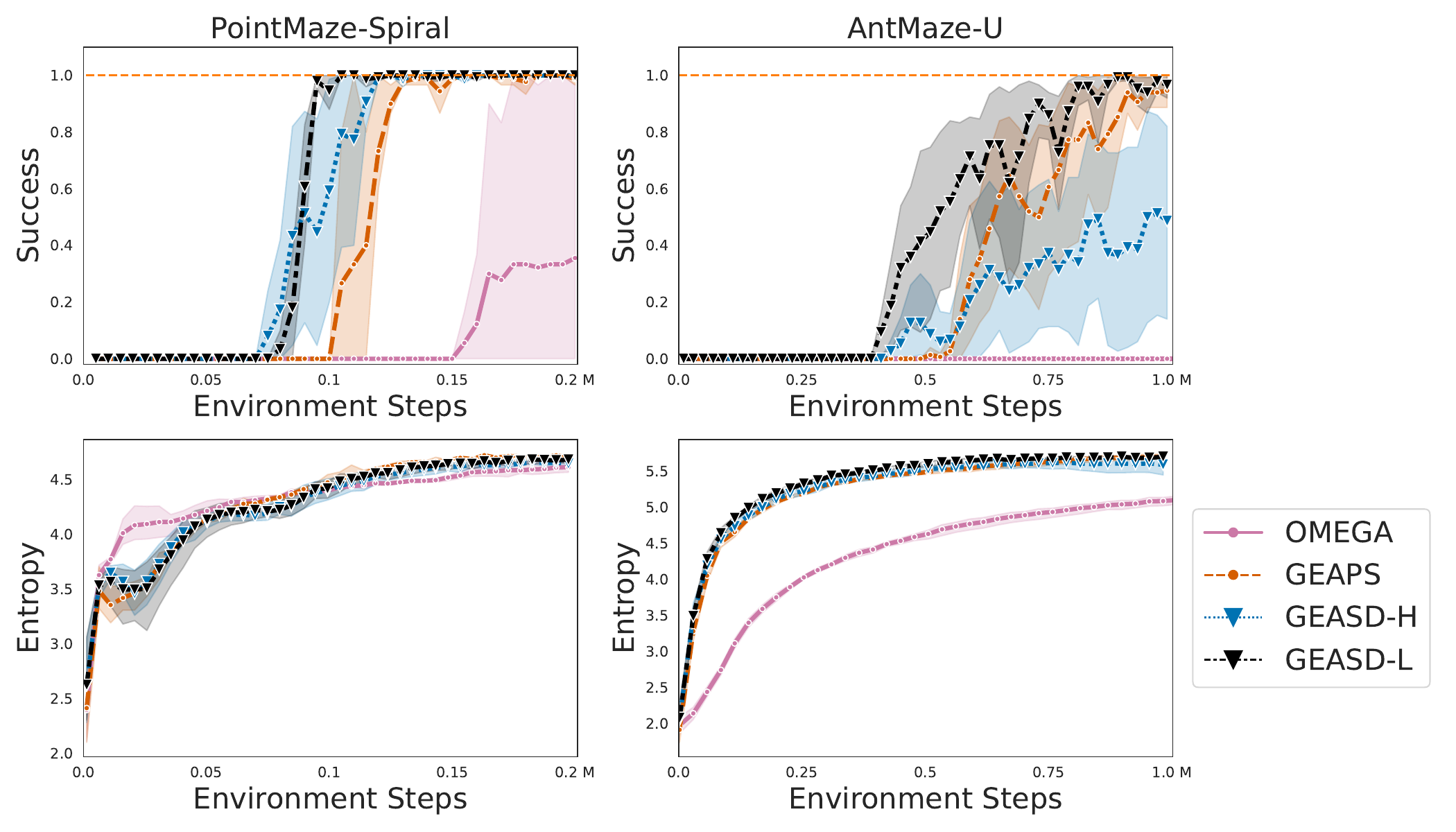}
    \caption{The test success on the desired goal distribution and the empirical entropy of achieved goals, throughout training on the two environments, for both the baselines and our models.}
    \label{fig:exp_comp_study}
\end{figure}

As illustrated in the first row of Fig.~\ref{fig:exp_comp_study}, our GEASD-L method demonstrates significant improvements in exploration efficiency over baseline methods in two challenging tasks. In \texttt{PointMaze-Spiral}, while the OMEGA method achieves only an average success rate of around 40\% at the end of training, skill-augmented exploration methods, including GEASD-L and GEASD-H, reach 100\% success rates. Notably, GEASD-H and GEASD-L achieve initial success 29\% and 24\% faster than GEAPS, respectively, and complete the tasks with full success 22\% and 8\% quicker than GEAPS. In \texttt{AntMaze-U}, OMEGA fails to achieve any success within one million training steps, whereas GEASD-L and GEASD-H reach initial success 19\% and 15\% faster than GEAPS. Throughout the training, only GEASD-L consistently achieves a 100\% success rate by approximately 900,000 steps across random seeds, while the success rates for GEASD-H and GEAPS plateau at about 50\% and 95\%, respectively. In comparison, GEASD-L reaches a 90\% success rate 20\% faster than GEAPS. As observed in Fig.~\ref{fig:exp_comp_study}, the learning process for GEASD-H is unstable in \texttt{AntMaze-U}, likely due to the limited availability of skill-level transition data, leading to significant variability in learning outcomes. These results underscore the efficiency of the GEASD methods in achieving earlier success compared to GEAPS, as well as the ability of GEASD-L to stabilize learning more effectively than GEASD-H and complete the tasks with success rates 100\% faster than those of GEAPS.

For the entropy of achieved goals, which estimates the coverage of the achieved goals, we show the empirical entropy of the achieved goals for the baselines and our methods in the second row of Fig.~\ref{fig:exp_comp_study}. In the \texttt{PointMaze-Spiral} environment, although OMEGA initially exhibits high entropy, it is surpassed by our skill-augmented methods, which achieve greater entropy through more effective exploration. Notably, GEASD-L and GEASD-H demonstrate marginal but consistent advantages over GEAPS. In the \texttt{AntMaze-U} setting, skill-augmented methods increase entropy significantly faster than OMEGA, with GEASD-L and GEASD-H again slightly outperforming GEAPS.

\begin{figure}[t]
    \centering
    \begin{subfigure}{\textwidth}
        \centering
        \includegraphics[width=\linewidth]{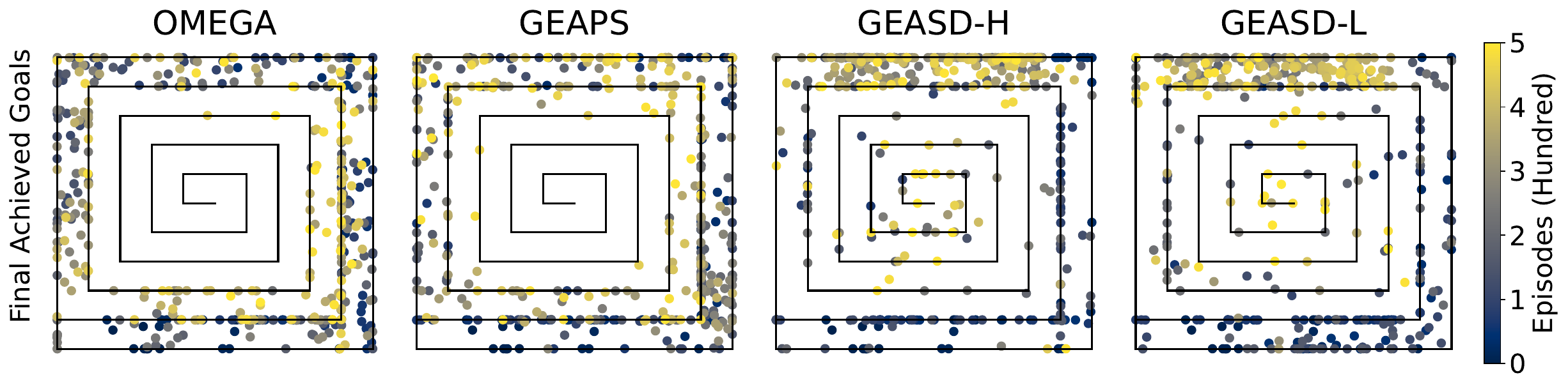}
        \caption{\texttt{PointMaze-Spiral}}
        \label{fig:point_maze_spiral_vis_heatmap}
    \end{subfigure}
    \vspace{1em}
    \begin{subfigure}{\textwidth}
        \centering
        \includegraphics[width=\linewidth]{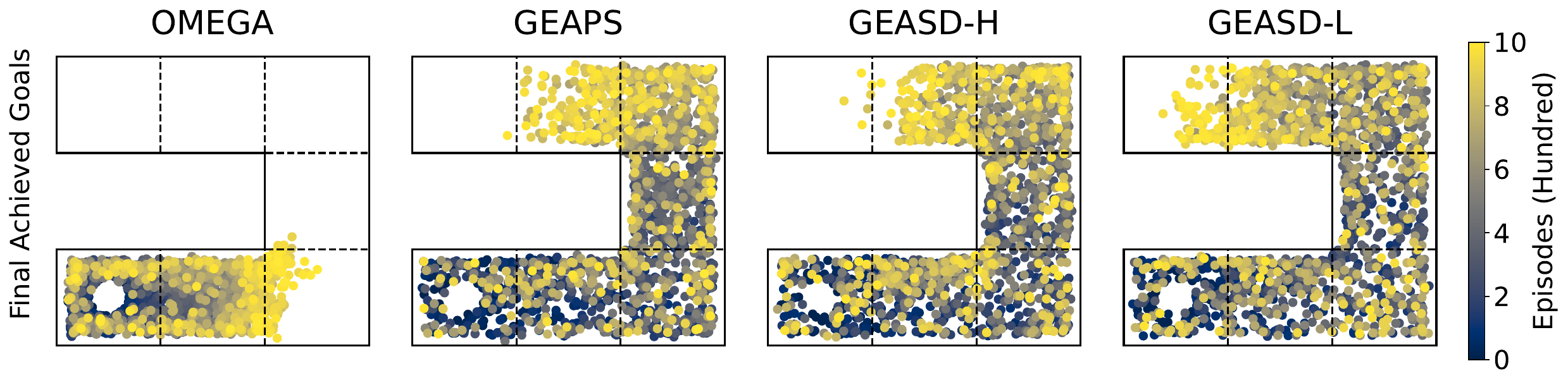}
        \caption{\texttt{AntMaze-U}}
        \label{fig:ant_maze_vis_heatmap}
    \end{subfigure}
    \caption{Visualization of the final goals achieved across historical episodes in (a)~\texttt{PointMaze-Spiral} and (b) \texttt{AntMaze-U}, with the training evolution process depicted through heatmaps.}
    \label{fig:visualization_heatmap}
\end{figure}

\subsubsection{Visualization of Exploration Progress}
To address the second question, we visualize the final goals achieved at the end of episodes, accumulated up to a certain timestep, to vividly illustrate the differences in exploration behaviors in Fig.~\ref{fig:visualization_heatmap}.

As depicted in Fig.~\ref{fig:point_maze_spiral_vis_heatmap}, none of the baselines have successfully explored the desired goal of \texttt{PointMaze-Spiral}, located in the center, within five hundred episodes. OMEGA has not fully explored the two outermost coils of the five in the spiral layout. Although GEAPS surpasses OMEGA through the augmentation of uniformly distributed skills, it has only just begun to extend its exploration to the third coil at the end of five hundred episodes. In contrast, both GEASD-H and GEASD-L have explored the locations of the goals within four hundred episodes only, showcasing the effectiveness of our adaptive distribution in adapting to the local structure and enhancing exploration efficiency.

In the case of \texttt{AntMaze-U}, which consists of seven grids as separated by dashed lines in Fig.~\ref{fig:ant_maze_vis_heatmap}, OMEGA has extended its exploration only to the third grid. GEAPS and GEASD-H have comparable goal coverage, extensively covering up to the sixth grid, but rarely explored the final grid where the desired goals are located. In contrast, GEASD-L has extensively explored the final grid. The lack of clear advantages in goal coverage by GEASD-H over GEAPS can be attributed to the sparsity of skill-level transition data, which hinders the agent's ability to efficiently leverage structural information. Conversely, GEASD-L benefits from low-level transition data, efficiently learning structural information to enhance goal coverage.

\subsubsection{Exploration Evaluation in Unseen Environments} \label{sect:explore_in_unseen}
The third research question is crucial as efficiency in the same environments may stem from merely reinforcing behaviors associated with the spread of locally achieved goals, rather than leveraging structural information.

To tackle the third question, we explore how our learned value-based skill distribution contrasts with those learned goal-conditioned policies in leveraging structural information in unseen environments that have similar or partially similar structures. In our evaluation, we discovered that while our chosen skills in \texttt{AntMaze-U} aid in exploration, they struggle to complete tasks independently, often becoming trapped in unrecoverable states.
Consequently, we focus our evaluation on environments similar to \texttt{PointMaze-Spiral}.

Specifically, we introduce two variations of \texttt{PointMaze-Spiral} as depicted in Fig.~\ref{fig:pointmaze_variants}: i) \texttt{PointMaze-Spiral-C}: a clockwise version of the spiral maze that demands a reversal in the direction of movement compared to \texttt{PointMaze-Spiral}. ii) \texttt{PointMaze-Serpentine}: a maze featuring horizontal serpentine paths that involve U-turns, which are unseen in the \texttt{PointMaze-Spiral}. Both variants share structural similarities with \texttt{PointMaze-Spiral} and require the agent to expend a comparable amount of time to learn and solve them from an initial state. For the convenience of further discussion, we will refer to these environments simply as \texttt{Spiral-C} and \texttt{Serpentine}, respectively, omitting the prefixes.
\begin{figure}[t]
    \centering
    \begin{subfigure}{0.24\textwidth}
        \centering
        \includegraphics[width=\linewidth]{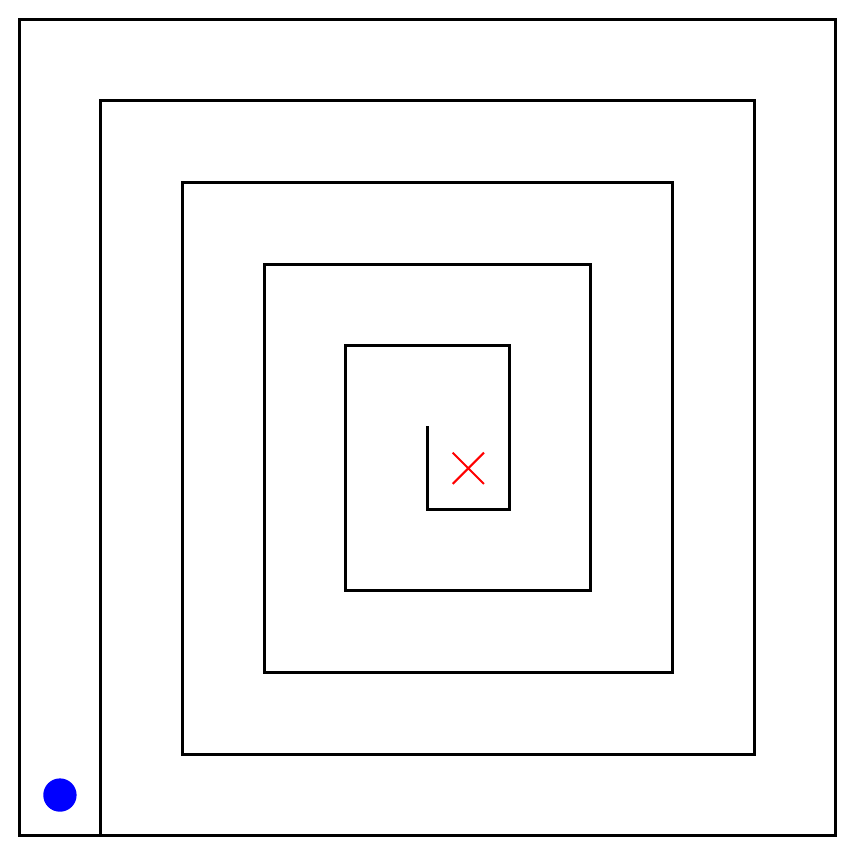}
        \caption{\texttt{Spiral-C}}
        \label{fig:inverse_spiral}
    \end{subfigure}
    \hspace{3em}
    \begin{subfigure}{0.24\textwidth}
        \centering
        \includegraphics[width=\linewidth]{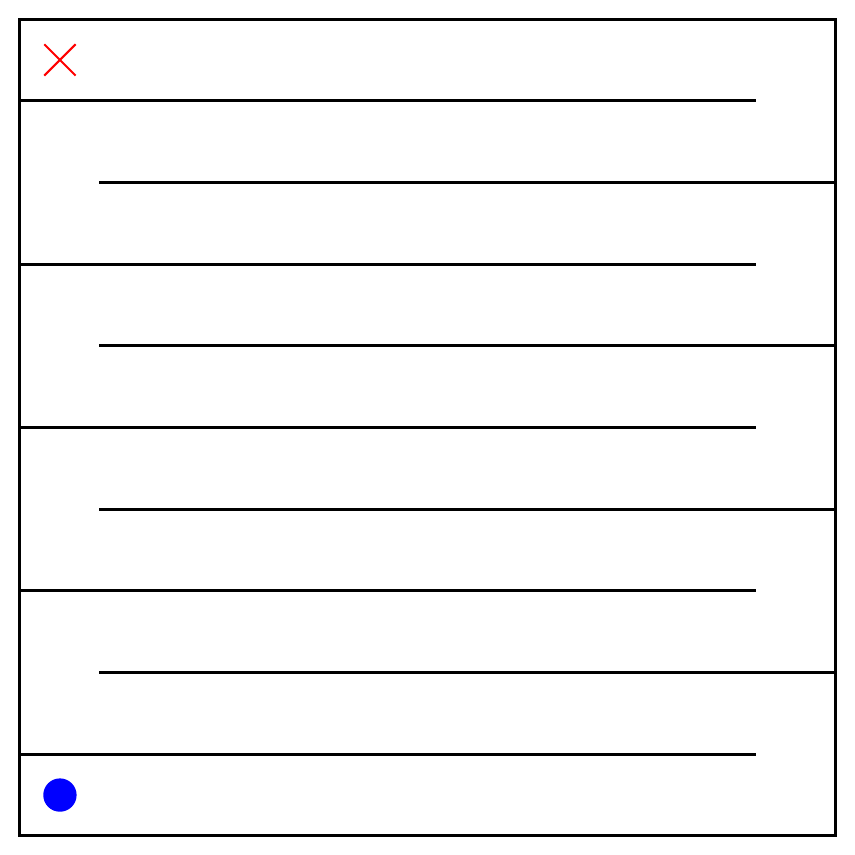}
        \caption{\texttt{Serpentine}}
        \label{fig:serpentine}
    \end{subfigure}
    \caption{Two variants of \texttt{PointMaze-Spiral}, where the blue circles denote the initilized positions and the red cross denote the goal positions.}
    \label{fig:pointmaze_variants}
\end{figure}

\begin{table}[ht]
\centering
\caption{Exploratory Performance in Unseen Variants of \texttt{PointMaze-Spiral}: \texttt{Spiral-C} and \texttt{Serpentine}.}
\label{table:generalization_test}
\begin{tabular}{l|ccc|ccc}
\toprule
& \multicolumn{3}{c|}{\texttt{PointMaze-Spiral-C}} & \multicolumn{3}{c}{\texttt{PointMaze-Serpentine}} \\
\midrule
\textbf{GC Policy} & \textit{SR} & \textit{Max Occ} & \textit{Avg Occ} & \textit{SR}  & \textit{Max Occ} & \textit{Avg Occ} \\
\midrule
GEASD-H & 0\% & 1\% & 1\% & 0\% & 11\% & 11\% \\
GEASD-L & 0\% & 2\% & 1.28\% & 0\% & 11\% & 10.94\% \\
GEAPS & 0\% & 1\% & 1.0\% & 0\% & 11\% & 7.46\%\\
OMEGA & 0\% & 2\% & 1.06\% & 0\% & 11\% & 11\%\\
Uniform  & 0\% & 18\% & 8.16\%& 0\% & 19\% & 7.48\% \\
\midrule
\textbf{Skill Policy} & \textit{SR} & \textit{Max Occ} & \textit{Avg Occ} & \textit{SR}  & \textit{Max Occ} & \textit{Avg Occ} \\
\midrule
GEASD-H & 54\%& 100\%& 94.12\% & 0\%& 70\%& 34.44\%\\
GEASD-L  & \textbf{90\%} & \textbf{100\%} & \textbf{97.96\%} & 0\% & \textbf{96\%} & \textbf{49.84\%}\\
Uniform  & 0\% & 34\% & 16.74\% & 0\% & 36\% & 15.9\% \\
\bottomrule
\end{tabular}
\end{table}

For the experiments on the two environments, we evaluate two types of policies derived from experimental methods: goal-conditioned (GC) policies and skill policies following a specific skill distribution, each including a uniform variant to serve as an unbiased benchmark for comparison. For GC policies, we directly evaluate their performance, conditioned on the desired goals. For the skill policies, we set the temperature \( T \) for the derivation of the skill distribution to a static value of $0.01$ to better exploit the learned structural information.

For the evaluation, we propose three metrics to assess performance: \emph{Success Rates} (\emph{Success}), \emph{Maximum Occupancy Ratio} (\emph{Max Occ}), and \emph{Average Occupancy Ratio} (\emph{Avg Occ}). Success rates are crucial for determining if the policies can discover the ultimate goals in new, unseen environments. However, since new environments often have different layouts, zero success rates alone do not fully capture how well the policies leverage structural information for exploration. Given that each environment consists of one hundred grids, the occupancy ratio—which measures the percentage of grids visited by the agent—serves as an effective indicator of exploration depth across the state space. \emph{Max Occ} represents the highest occupancy ratio achieved, highlighting the best exploratory efforts, while \emph{Avg Occ} calculates the average exploratory performance, offering insights into consistent exploratory behavior.
These metrics are calculated from fifty independent episodes in these environments,\textit{ without any further learning during these interactions.} It ensures that the performance metrics accurately reflect the capabilities of the policies at the end of the training process on \texttt{PointMaze-Spiral}.

As shown in Table~\ref{table:generalization_test}, GC policies fail to successfully reach the target grid in both environments, likely due to biased learning outcomes. As shown in the trajectory visualization in Appendix~\ref{sect:supp_exps}, we observe that the learned GC policies ignore the changes in the environmental structures and stick to move along the  learned path in the training task \texttt{PointMaze-Spiral}, leading at most 2\% and 11\% Max Occ in \texttt{Spiral-C} and \texttt{Serpentine}, respectively. In contrast, the uniformed policy has no bias in choose the actions, leading to better Max Occ as 18\% and 19\%, respectively. It means uniform policies potentially explore the environments more extensively than the learned GC policies. Combining the evaluations on both environments, the performances of GC policies vary more significantly than the uniform policy, which also shows the learned bias can potentially impede the exploration process.

In contrast, both GEASD-L and GEASD-H are capable of achieving significant exploration progress in both environments. Unlike the zero success of GC policies on \texttt{Spiral-C}, GEASD-H and GEASD-L achieve 54\% and 90\% success in reaching the desired grid, respectively. They both reach 100\% Max Occ and over 94\% Avg Occ, providing valuable data for the agent to learn goal-conditioned behaviors. In \texttt{Serpentine}, though they do not achieve any success, GEASD-H and GEASD-L reach 70\% and 96\% Max Occ, respectively, which are improvements of over  535\% and 770\% over the learned GC policies. Their Avg Occ also improves the best statistics of GC policies by  210\% and 350\%, respectively. Although the uniform skill policy performs better than the uniform GC policy given the occupancy metrics, it falls behind GEASD-H and GEASD-L by a large margin. From this comparison, the significant advantages of the GEASD methods are not solely attributed to the introduction of skills, which underscores the effectiveness of our learned skill distribution in leveraging the structural information.

In summary, the learned GC policies, which ignore structural information and base decisions solely on the absolute positions within the mazes, tend to overfit to the training environment. In contrast, our learned skill distribution, based on skill value functions, effectively leverages structural information. Even when faced with unseen structural challenges like the U-turns in \texttt{Serpentine}, leveraging the remaining known structural information can significantly enhance exploration efficiency, as evidenced by the statistics.

\subsubsection{Ablation Study}
To address the fourth research question, we conduct an ablation study focusing on a set of key experimental hyperparameters in the GEASD methods. Due to the possibility of GEASD-H encountering data inadequacy problems, we narrow our focus of the ablation study to GEASD-L only. Specifically, we explore: i) the impact of utilizing static temperatures for deriving the skill distribution, in comparison to our dynamic temperature design $T_{\text{dyna}}$. 
ii) the significance of including action history within the historical context. As the overall entropy does not differ significantly, as shown in the comparative study in Section~\ref{sect:comp_study}, we concentrate solely on the success rates on the desired goal distribution in the ablation study. For more experiments pertaining to the ablation study, see Appendix~\ref{sect:supp_exps}.

\begin{figure}[t]
    \centering
    \begin{subfigure}{0.9\linewidth}
        \includegraphics[width=\textwidth]{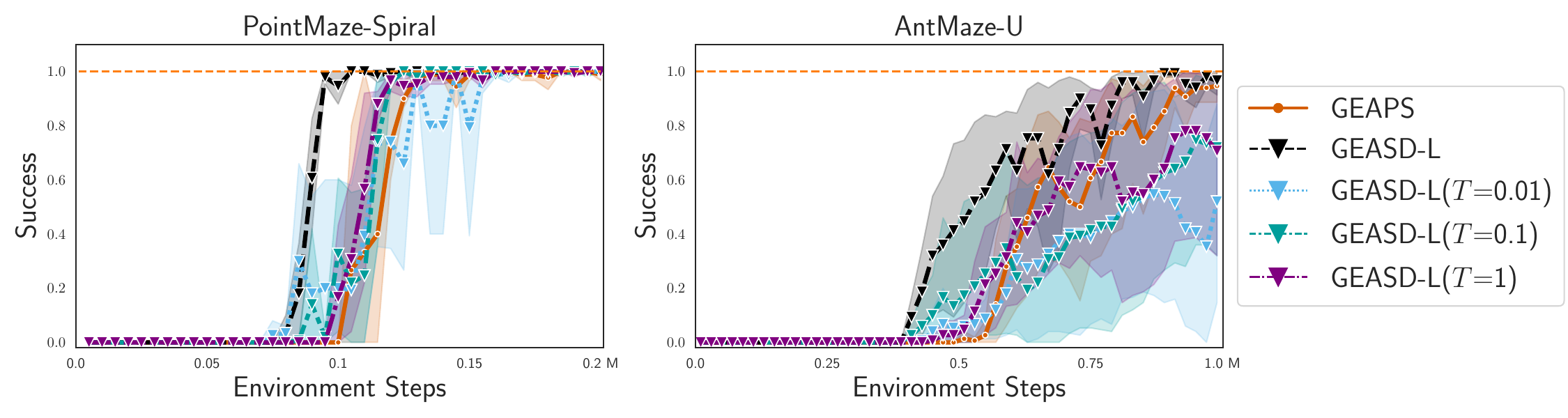}
        \caption{Dynamic vs. Static Temperature}
        \label{fig:ablation_study_temp}
    \end{subfigure}
    \centering
    \begin{subfigure}{0.9\linewidth}
        \includegraphics[width=\textwidth]{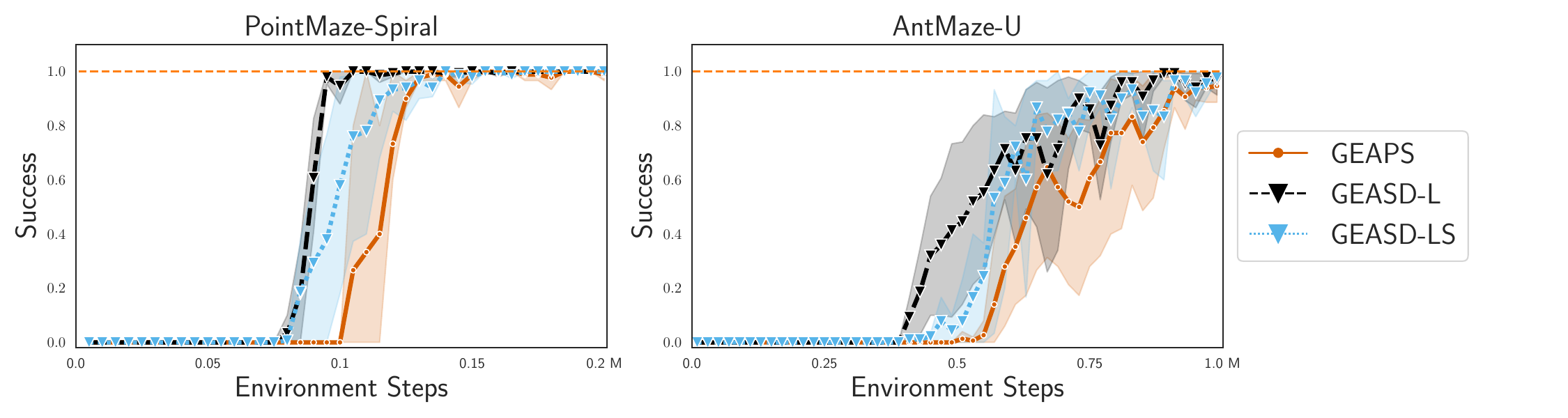}
        \caption{With vs. Without Action History}
        \label{fig:ablation_study_history}
    \end{subfigure}

    \caption{A comparative ablation study on achieving desired goal distribution: analyzing the impact of temperature settings and action history inclusion on success rates throughout training.}
\end{figure}

\paragraph{Dynamic vs. Static Temperature}
In our ablation study focused on the role of temperature in skill derivation, we introduced three experiments with skill distributions derived using static temperatures of $0.01$, $0.1$, and $1$. These temperatures were selected to provide even intervals on a logarithmic scale from $0.01$ to $1$, reflecting the varying scope of our dynamic temperature $T_\text{dyna}$. We refer to these experimental settings as GEASD-L($T=0.01$), GEASD-L($T=0.1$), and GEASD-L($T=1$). As depicted in Fig.~\ref{fig:ablation_study_temp}, all GEASD methods, whether employing static or dynamic temperatures, achieved earlier initial success compared to the GEAPS approach. This suggests that our methods might more effectively aid the agent in faster goal discovery than the uniform skill distribution used in GEAPS.
Among the variations, our GEASD-L with dynamic temperature achieves superior learning efficiency, reaching 100\% success rates faster than all other variants that employ static temperatures. Notably, GEASD-L($T=0.01$), which predominantly exploits skills with higher values and explores skills with lower values less frequently, exhibits the poorest performance in terms of final success rates and learning speed. This trend underscores that focusing primarily on exploiting the skills with the highest values, which may be inaccurate, can hinder the necessary exploration required to acquire higher-quality data.
On the other extreme, GEASD-L($T=1$) achieves initial success latest among the three variants, indicating its inefficiency in exploiting structural information to discover the desired goals. However, among the three static temperature variants, it converges to 100\% success on \texttt{PointMaze-Spiral} at the fastest rate and achieves the highest success rate of 80\% on \texttt{AntMaze-U}. This phenomenon is likely because its disadvantages in exploiting structural information are compensated by increased exploration trials, whereas the other two variants suffer from poorer data coverage due to their heavy reliance on skill value functions, which can be inaccurate in some scenarios.
The learning trend of GEASD-L($T=1$) most closely resembles that of GEAPS among the three variants, confirming our assumption that higher temperatures lead to a skill distribution that more closely mirrors the uniform distribution. The performance of GEASD-L($T=0.1$) lies between the two extremes. Overall, the ablation study highlights the advantages of $T_\text{dyna}$ over static temperatures and further substantiates our assumptions regarding the impact of temperature on skill distribution to a considerable extent.

\paragraph{With vs. Without Action History}
In our ablation study, which focuses on the role of action history in the historical context for skill derivation, we introduce a variant of GEASD-L that includes only state history in the historical context, denoted as GEASD-LS.
As illustrated in Fig.~\ref{fig:ablation_study_history}, the inclusion of action history enhances learning efficiency, particularly in the initial stages of learning in both environments. Specifically, GEASD-L reaches 95\% success in \texttt{PointMaze-Spiral} and over 40\% success in \texttt{AntMaze-U} approximately 17\% faster than GEASD-LS in both environments. After reaching these thresholds, the learning trends of both variants align closely within each environment. The improvement can be attributed to the fact that action history serves as a more reliable source of information in the early stages, potentially elucidating the relationships among historically achieved goals when the data are insufficient for the agent to discern such relationships from the state history alone.
The reliability of action history may stem from each action's association with specific semantic changes in state transitions, such as moving in a direction by a certain distance, which implicitly indicates the distances between achieved goals. Therefore, the ablation study highlights the advantages of including action history in addition to state history in the historical context.

\section{Discussion}
\label{sect:discussion}
In this section, we discuss the limitations and challenges associated with our work and establish links to other related studies. While our approach has demonstrated several advantages, it also presents unresolved issues and limitations.

\paragraph{Applicability to Vision-Based Scenarios}
For instance, in our case, empirical entropy is estimated by modeling the distribution of achieved goals using Gaussian models, as proposed by OMEGA \citep{pitis2020maximum}. However, this method is not directly applicable to purely vision-based scenarios. For these contexts, embedding-based methods for intra-episode novelty \citep{savinov2018episodic, badia2020never} may offer a viable solution. Here, the primary focus shifts from local entropy changes to variations in distances that measure the dispersion of learned embeddings within a contextual horizon.

\paragraph{Context-Dependent Skill Horizon}
In our study, the skill horizon \( k \) is fixed. Ideally, the optimal skill horizon should correlate with the extent of environmental information that can be perceived. On one hand, if an agent has only enough structural information to make precise estimations about local exploration changes over a range shorter than \( k \), it might struggle to accurately forecast these changes for the subsequent \( k \) steps. On the other hand, if the skill horizon is too narrow relative to the breadth of perceived information, the agent would be able to leverage only a small portion of this available information. Thus, the optimal \( k \) should be context-dependent, and determining how to derive such a \( k \) will be a key focus of future work.

\paragraph{Adapting Entropy Estimation to Account for Sequential Skill Execution}
Considering pre-set skills, rather than simply evaluating local entropy changes from executing a single skill, we propose exploring the entropy changes resulting from the sequential execution of various skills in future studies. This prospective adaptation aims to account for more complex behaviors using the same skill sets by assessing local entropy changes in a forward-thinking manner, considering not just the current skill but also potential future skills. Additionally, it could enable us to leverage structural information more effectively and intricately.

\vspace{1em}

Overall, there are several potential directions to further enhance the exploration efficiency of GEASD in complex environments, which will be explored in future work.

\section{Conclusion}
In this paper, we propose a novel learning objective that enhances the overall entropy of achieved goals by optimizing local entropy within a contextual horizon. This optimization enables the agent to exploit structural information to derive behavior patterns that not only enhance local entropy but may also generalize to structurally familiar but previously unseen scenarios.

We demonstrate our ideas through an adaptive skill distribution based on a predefined skill set. Utilizing theoretical analysis with certain assumptions, we justify the adoption of a Boltzmann distribution for skill distribution, which is grounded on skill value functions that estimate the local entropy changes induced by skill execution. Our experiments on challenging sparse-reward, long-horizon tasks show that augmenting goal exploration with skills from our adaptive distribution significantly improves exploration efficiency. Furthermore, in scenarios where the skills are adequate for task completion, the skill policy based on the adaptive distribution achieves substantially better initial exploration progress in new environments, a feature not present in learned goal-conditioned policies or uniform skill policies. Our ablation study highlights the benefits of dynamic temperature settings in the skill distribution derivation and underscores the significant role of action history in this context. The results indicate that leveraging structural information to enhance local entropy is a viable approach to efficiently tackle sparse-reward and long-horizon tasks.

Additionally, we discuss potential areas for future research, including applications to vision-based tasks, exploring optimal context-dependent skill horizons, and  adapting the entropy estimations to account for sequential skill execution, as outlined in Section~\ref{sect:discussion}.

In conclusion, our approach to goal exploration via adaptive skill distribution (GEASD) advances the field of goal-conditioned reinforcement learning. We believe our findings pave the way for novel exploration strategies that effectively tackle challenging tasks characterized by sparse rewards and long horizons, particularly those involving repeated structures.

\appendix
\section*{Appendix}
\setcounter{proposition}{0}

\numberwithin{equation}{section}
\counterwithin{figure}{section}
\counterwithin{table}{section}
\renewcommand{\thefigure}{\thesection\arabic{figure}}
\renewcommand{\thetable}{\thesection\arabic{table}}
\renewcommand{\theequation}{\thesection\arabic{equation}} 
\section{Proof of Propositions} \label{appendix:proof_propositions}
In this appendix, we provide proofs for the propositions formulated in Section~\ref{sect:theoretical_analysis} of the main
text.

\begin{proposition} \label{prop:boltzamann_distribution_with_proof}
Consider a set of skills denoted by \( \mathcal{Z} \), where each skill \( \bm{z}_i \in \mathcal{Z} \) uniquely covers a portion of the achieved goals, thereby ensuring \( H(\mathcal{Z}|s_t, \bm{\Phi}(h^k_{t+k})) = 0 \). The optimal skill distribution, conditioned on the historical trajectory \( h^C_t \) and aiming to maximize the entropy \( H(\bm{\Phi}(h^C_{t+k})|h^C_t) \), conforms to a form of the Boltzmann distribution. This distribution is expressed as
\begin{equation}
p^*(\bm{z}|h^C_t) = \frac{e^{-E^*(\bm{z})}}{\sum_i e^{-E^*(\bm{z}_i)}} = \frac{e^{\Delta H(\bm{\Phi}(h^C_{t+k})|h^C_t, \bm{z})}}{\sum_i e^{\Delta 
 H(\bm{\Phi}(h^C_{t+k})|h^C_t, \bm{z}_i)}},  \nonumber
\end{equation}
where the energy function \(E^*(\bm{z})\) associated with skill \( z \) is quantified by the resulting entropy gain after execution, \( - \Delta H(\bm{\Phi}(h^C_{t+k})|h^C_t, \bm{z}) \).
\end{proposition}

\begin{proof}
We begin by analyzing the entropy of the updated goal space conditioned on the historical trajectory \( h^C_t \):
\begin{align}
    H(\bm{\Phi}(h^C_{t+k})|h^C_t) &=  H(\bm{\Phi}(h^C_{t+k})|h^C_t, \mathcal{Z}) + H(\mathcal{Z}|h^C_t) - H(\mathcal{Z}|h^C_t, \bm{\Phi}(h^C_{t+k})) \label{eq:entropy_history_conditioned} \\
    &= H(\bm{\Phi}(h^C_{t+k})|h^C_t, \mathcal{Z}) + H(\mathcal{Z}|h^C_t) - H(\mathcal{Z}|h^C_t, \bm{\Phi}(h^k_{t+k}))  \label{eq:simplify_to_explore} \\
    &= H(\bm{\Phi}(h^C_{t+k})|h^C_t, \mathcal{Z}) + H(\mathcal{Z}|h^C_t) - H(\mathcal{Z}|s_t, \bm{\Phi}(h^k_{t+k})) \label{eq:simplify_to_state} \\
    &= \sum_i p(\bm{z}_i|h^C_t) \left( H(\bm{\Phi}(h^C_{t+k})|h^C_t, \bm{z}_i) - \log p(\bm{z}_i|h^C_t) \right) \label{eq:entropy_history_expansion}
\end{align}

The transition from Eq.~(\ref{eq:entropy_history_conditioned}) to Eq.~(\ref{eq:simplify_to_explore}) relies on the fact that information pertaining to \( \bm{\Phi}(h^C_t) \) is encapsulated within \( h^C_t \), thereby allowing for a simplification of \( \bm{\Phi}(h^C_{t+k}) \) to \( \bm{\Phi}(h^k_{t+k}) \). Furthermore, the progression from Eq.~(\ref{eq:simplify_to_explore}) to Eq.~(\ref{eq:simplify_to_state}) is based on the fact that \( s_t \), as included in \( h^C_t \), along with \( \bm{\Phi}(h^k_{t+k}) \), together suffice to determine \( \mathcal{Z} \), as indicated by \( H(\mathcal{Z}|s_t, \bm{\Phi}(h^k_{t+k})) = 0 \).
Additionally, historical information preceding time step \( t \) does not introduce further uncertainties into skill selection because the skill execution, characterized by \( z \) and represented as \( a_t \sim \sigma(\cdot|\psi(s_t), \bm{z}) \), operates independently of past history. Ultimately, Eq.~(\ref{eq:entropy_history_expansion}) represents our primary optimization goal. With the constraint \( \sum_i p(\bm{z}_i|h^C_t) = 1 \) and the introduction of a Lagrangian multiplier \( \lambda \) (\(\lambda > 0\)), we formulate the optimization problem as:
\begin{equation}
    \mathcal{L} = \sum_i p(\bm{z}_i|h^C_t) \Bigl( H(\bm{\Phi}(h^C_{t+k})|h^C_t, \bm{z}_i) - \log p(\bm{z}_i|h^C_t) \Bigr) - \lambda \Bigl( \sum_i p(\bm{z}_i|h^C_t) - 1 \Bigr)\nonumber.
\end{equation}

Differentiating with respect to \( p(\bm{z}_i|h^C_t) \) gives:
\begin{align}
  H(\bm{\Phi}(h^C_{t+k})|h^C_t, \bm{z}_i) - \log p^*(\bm{z}_i|h^C_t) - 1 &= \lambda, \nonumber\\
  p^*(\bm{z}_i|h^C_t) &= e^{H(\bm{\Phi}(h^C_{t+k})|h^C_t, \bm{z}_i) - \lambda - 1} \nonumber.
\end{align}
Applying the normalization condition \( \sum_i p^*(\bm{z}_i|h^C_t) = 1 \), we find:
\begin{equation}
    \lambda = \log \left(\sum_i e^{H(\bm{\Phi}(h^C_{t+k})|h^C_t, \bm{z}_i) - 1}\right) \nonumber.
\end{equation}
Therefore, the optimal skill distribution \( p(\bm{z}_i|h^C_t) \) is derived as:
\begin{align}
    p^*(\bm{z}_i|h^C_t) &= \frac{e^{H(\bm{\Phi}(h^C_{t+k})|h^C_t, \bm{z}_i) - 1}}{\sum_i e^{H(\bm{\Phi}(h^C_{t+k})|h^C_t, \bm{z}_i) - 1}} \nonumber\\
    &= \frac{e^{H(\bm{\Phi}(h^C_{t+k})|h^C_t, \bm{z}_i) -H(\bm{\Phi}(h^C_t)|h^C_t) }}{\sum_i e^{H(\bm{\Phi}(h^C_{t+k})|h^C_t, \bm{z}_i)- H(\bm{\Phi}(h^C_t)|h^C_t) }}\nonumber\\
    &= \frac{e^{\Delta H(\bm{\Phi}(h^C_{t+k})|h^C_t, \bm{z})}}{\sum_i e^{\Delta 
 H(\bm{\Phi}(h^C_{t+k})|h^C_t, \bm{z}_i)}}\nonumber.
\end{align}
\end{proof}

\begin{proposition}\label{prop:expect_lowerbound_with_proof}
The local entropy change, \(\Delta H(\bm{\Phi}(h^C_{t+k})|h^C_t, \bm{z})\), resulting from executing skill \(\bm{z}\), is lower bounded by the expected change in local entropy across each trajectory \(h^k_{t+k}\) sampled from \(p(h^k_{t+k}|s_t, \bm{z})\). This relationship is formally described by the following inequality:
\begin{equation}
    \Delta H(\bm{\Phi}(h^C_{t+k})|h^C_t, \bm{z}) \geq \mathbb{E}_{h^k_{t+k} \sim p(\cdot|s_t, \bm{z})} \left[ H(\bm{\Phi}(h^C_{t+k})|h^C_{t+k}) - H(\bm{\Phi}(h^C_t)|h^C_t) \right]. \nonumber
\end{equation}
\end{proposition}

\begin{proof}
The distribution of goals within the contextual horizon \(C\), subsequent to the trajectory segment \(h^k_{t+k}\) and the execution of skill \(\bm{z}\) from the historical context \(h^C_t\), is represented by \(p(g|h^C_{t+k})\) and \(p(g|h^C_t, \bm{z})\), respectively. These distributions are articulated as follows:
\begin{align}
    p(g|h^C_{t+k}) &= \frac{k \cdot p(g| h^k_{t+k}) + (C-k) \cdot p(g|h^{C-k}_t)}{C}, \nonumber\\
    p(g|h^C_t, \bm{z}) &= \mathbb{E}_{h^k_{t+k} \sim p(\cdot|s_t, \bm{z})} \left[ p(g|h^C_{t+k}) \right]. \nonumber
\end{align}
Upon decomposing \(H(\bm{\Phi}(h^C_{t+k})|h^C_t, \bm{z})\), we obtain:
\begin{align}
  H(\bm{\Phi}(h^C_{t+k})|h^C_t, \bm{z}) &= -\sum_g p(g|h^C_t, \bm{z}) \log p(g|h^C_t, \bm{z}) \label{eq:entropy_decomposition}\\
                                 &\geq -\mathbb{E}_{h^k_{t+k} \sim p(\cdot|s_t, \bm{z})}  \left[\sum_g p(g|h^C_{t+k}) \log p(g|h^C_{t+k})\right] \label{eq:entropy_lower_bound} \\
                                 &= \mathbb{E}_{h^k_{t+k} \sim p(\cdot|s_t, \bm{z})}  \left[ H(\bm{\Phi}(h^C_{t+k})|h^C_{t+k}) \right] \label{eq:entropy_lower_bound_simplification}.
\end{align}
The derivation from Eq.~(\ref{eq:entropy_decomposition}) to Eq.~(\ref{eq:entropy_lower_bound}) leverages Jensen's inequality and the inherent concavity of the entropy function. Consequently, by integrating Eqs.~(\ref{eq:entropy_gain_simp}) and (\ref{eq:entropy_lower_bound_simplification}), it is straightforward to deduce:
\begin{equation}
    \Delta H(\bm{\Phi}(h^C_{t+k})|h^C_t, \bm{z}) \geq \mathbb{E}_{h^k_{t+k} \sim p(\cdot|s_t, \bm{z})} \left[ H(\bm{\Phi}(h^C_{t+k})|h^C_{t+k}) - H(\bm{\Phi}(h^C_t)|h^C_t)\right]. \nonumber
\end{equation}
\end{proof}

\section{Comprehensive Details of Methods}
\label{sect:details}
\subsection{Technical Details}
\subsubsection{OMEGA}
To ensure that the agent effectively reaches the desired goals from the distribution \( p_{\text{dg}} \), \cite{pitis2020maximum} aims to minimize the KL divergence between \( p_{\text{dg}} \) and the distribution of achieved goals \( p_{\text{ag}} \). This leads to the following objective:
\begin{equation}
  J(p_{\text{ag}}) = D_{KL}(p_{\text{dg}}||p_{\text{ag}}). \nonumber
\end{equation}

However, this optimization objective is not finite when \( p_{\text{ag}} \) and \( p_{\text{dg}} \) do not initially overlap. To address this issue, the \textit{Maximum Entropy Goal Achievement} (MEGA) modifies the objective by expanding the support of achieved goals, \( \text{supp}(p_{\text{ag}}) \). MEGA specifically aims to maximize the entropy of the achieved goals as shown in the following objective:
\begin{equation}
  J_{\text{MEGA}}(p_{\text{ag}}) = D_{KL}(\mathcal{U}(\text{supp}(p_{\text{ag}}))||p_{\text{ag}}). \nonumber
\end{equation}
Here, \( \mathcal{U}(\text{supp}(p_{\text{ag}})) \) denotes the uniform distribution over the support of \( p_{\text{ag}} \). MEGA, however, neglects the distribution of desired goals in its formulation. In contrast, OMEGA introduces a gradual transition from the uniform distribution over the support of the achieved goals to the distribution of desired goals, based on current exploration progress. This approach utilizes a hyper-parameter \( \alpha \) to gauge this progress, calculated as follows:
\begin{equation}
  \alpha = \frac{1}{\max(b + D_{KL}(p_{\text{dg}} || p_{\text{ag}}), 1)},
  \label{eq:omega-alpha-calculation}
\end{equation}
where \( b \leq 1 \). Using this parameter, OMEGA defines a mixture distribution \( p_{\alpha} = \alpha p_{\text{dg}} + (1-\alpha)\mathcal{U}(\text{supp}(p_{\text{ag}})) \) as the target for optimizing the KL divergence:
\begin{equation}
  J_{\text{OMEGA}}(p_{\alpha}) = D_{KL}(p_{\alpha}||p_{\text{ag}}). \nonumber
\end{equation}
As \( \alpha \) progresses from nearly zero to one, the target of the KL divergence shifts from predominantly \( \mathcal{U}(\text{supp}(p_{\text{ag}})) \) to \( p_{\text{dg}} \) entirely. To optimize the OMEGA objective, \cite{pitis2020maximum} has demonstrated that the optimal sub-goal selection strategy involves sampling the sub-goal from the desired goals at an \( \alpha \)-probability and from the achieved goals at a \( 1-\alpha \)-probability as follows:
\begin{equation}
  \hat{g} = \argmin_{\hat{g}\in\mathcal{B}} p_{\text{ag}}(\hat{g}).\label{eq:omega-min-density} \nonumber
\end{equation}

\subsubsection{GEAPS} \label{sect:geaps_technical_details}
GEAPS, introduced by \citep{wu2024goal}, presents a goal-exploration framework that includes OMEGA as a specific example. This framework categorizes the sub-goal selection-based exploration methods into two stages: goal pursuit and goal exploration. Goal pursuit aims to reach the selected sub-goal, while goal exploration initiates after the achievement of the sub-goal to conduct goal-independent exploration. As discussed in Section~\ref{sect:exploration_obj}, GEAPS aims to maximize the entropy of $H(\bm{\Phi}(h^k_{t+k}))$, specifically during the goal exploration stage where the exploration behaviors are not tied to any goals.

To achieve this, GEAPS focuses on skill learning to capture the compositions of goal-transition patterns shared between the pre-training tasks and the downstream tasks. The goal-transition patterns are defined as effective transition behavior patterns between two achieved goals, circumventing ineffective actions in the environments, thereby promoting the entropy of achieved goals within the skill horizon.

The agent is expected to cover as many goals as possible in the pre-training environments, aiming to capture more goal-transition patterns so that they can act effectively even when some are missing in downstream task scenarios. The optimization objective is to learn distinct skills while also optimizing the coverage of achieved goals by executing skills from a uniform distribution. To learn such a discrete set of skills, the skill learning objective to maximize is set as follows:
\begin{equation}
J(\mathcal{G}, \mathcal{Z}) = \argmax_{\mathcal{Z}}H(\mathcal{G}) = \argmax_{\mathcal{Z}}\left[I(\mathcal{G}, \mathcal{Z}) + H(\mathcal{G}|\mathcal{Z})\right]. \nonumber
\end{equation}
To optimize the objective, GEAPS proposes intrinsic rewards for the skill characterized by $\bm{z}$ at state $s_t$ as follows:
\begin{equation}
r_z(s_t) =
\underbrace{\log q(\bm{z} \mid \phi(s_t)) - \log p(\bm{z}) \vphantom{\left(\max_{\hat{g}} q(\hat{g} \mid \bm{z}\right)}}_{\mathclap{\text{optimize } I(\mathcal{G}, \mathcal{Z})}}
+ \beta
\underbrace{\left( \log \max_{\hat{g}} q(\hat{g} \mid \bm{z}) - \log q(\phi(s_t) \mid \bm{z}) \right)}_{\mathclap{\text{optimize } H(\mathcal{G}|\mathcal{Z})}}, \nonumber
\end{equation}
where the right-hand side can be decomposed into two parts as indicated by the underbraces: the first part aims to optimize the mutual information term $I(\mathcal{G}, \mathcal{Z})$, and the second part aims to optimize the conditional entropy term $H(\mathcal{G}|\mathcal{Z})$. Here, $\beta$ manages the trade-off between the two parts of the rewards.

\subsection{Implementation Details}
\subsubsection{DDPG}
All goal-conditioned policies are based on the Deep Deterministic Policy Gradient (DDPG) framework \citep{lillicrap2015continuous}. The specific hyperparameters employed in our DDPG implementation are listed in Table~\ref{table:ddpg-hyperparameters}.

\begin{table}[h!]
\caption{Hyperparameters used in the DDPG implementation.}
\centering
\begin{tabular}{c | c}
 \hline
 \textbf{Hyperparameter} & \textbf{Value} \\ [0.5ex]
 \hline\hline
batch size & 2000 (1000 for \texttt{AntMaze-U}) \\
actor learning rate & 0.001 \\
critic learning rate & 0.001 \\
optimizer & Adam \citep{kingma2014adam} \\
activation function & GeLU \citep{hendrycks2016gaussian} \\
hidden layer sizes (actor and critic) & (512, 512, 512) \\
target network update proportion & 5\% \\
target network update frequency & every 40 steps \\
initial random data collection & 5000 steps \\
epsilon for random exploration & 0.1 \\
replay buffer size & 5,000,000 \\
discount factor & 0.98 (0.99 for \texttt{AntMaze-U}) \\  [1ex]
 \hline
\end{tabular}
\label{table:ddpg-hyperparameters}
\end{table}

\subsubsection{OMEGA}
In all our experiments, data collection for the training process is conducted simultaneously across seven instances of the same environment, each running on parallel threads. We adopt the same configurations as used in OMEGA \citep{pitis2020maximum}. The $b$ parameter in Eq.~\eqref{eq:omega-alpha-calculation} is set to -3.0. We employ a kernel density estimator (KDE) \citep{rosenblatt1956remarks} with a 0.1 bandwidth and a Gaussian kernel to estimate the density of the achieved goals. This KDE is fitted to 10,000 normalized achieved goals from the replay buffer at each optimization step. The empirical entropy metric used in the Comparative Study, as outlined in Section~\ref{sect:comp_study}, is estimated using 1,000 samples drawn from the KDE, along with their estimated probabilities.

\subsubsection{GEAPS}
In our experiments, we employ the pre-trained skills provided in the GEAPS codebase \citep{wu2024goal}. These skills uniquely utilize the agent's inner state as inputs, deliberately excluding environmental details such as wall information and absolute spatial coordinates. Consequently, while these skills are more adaptable and can generalize across various scenarios, they do not leverage the environmental structural information. The experimental settings closely follow those of OMEGA, except for an additional skill augmentation phase during the goal exploration stage.

\subsubsection{GEASD}
Building on the pre-trained skills utilized in GEAPS, we employ a Gated Recurrent Unit (GRU) \citep{cho-etal-2014-learning} to learn skill value functions. Data are sampled from the same replay buffer as used for goal-conditioned strategies, although GEASD-H requires the exclusion of non-skill execution data to effectively learn these functions. The hyperparameters for the GRU architecture and its training are comprehensively detailed in Table~\ref{table:gru-hyperparameters}.

Beyond the standard parameters listed in Table~\ref{table:gru-hyperparameters}, we have introduced a clipping ratio. This ratio determines the probability that the historical context input is randomly truncated through masking. Applying this mechanism enables the agent to effectively utilize a shorter contextual horizon, which is particularly beneficial when the available context does not span the entire horizon. Furthermore, we anticipate that the agent will leverage a smaller context for skill value function inference, especially when the full context contains unfamiliar structural information to the agent. Throughout our research, we have maintained the clipping ratio at 0.5.

\begin{table}[h!]
\caption{Hyperparameters for the GRU model to model the skill value functions.}
\centering
\begin{tabular}{c | c}
 \hline
 \textbf{Hyperparameter} & \textbf{Value} \\ [0.5ex]
 \hline\hline
batch size & 64 \\
learning rate & 1e-3 \\
optimizer & Adam \citep{kingma2014adam} \\
activation & GeLU \citep{hendrycks2016gaussian} \\
rnn hidden layer size & 128 \\
rnn output layer size & 128 \\
input hidden layer sizes & (64, 64) \\
output hidden layer sizes & (64, 64) \\
update frequency & every 4 steps \\
\hline
\end{tabular}
\label{table:gru-hyperparameters}
\end{table}

\section{Supplementary Visualizations and Ablation Studies}
\label{sect:supp_exps}
\subsection{Visualization Analysis of Policy Performance in Unseen Environments}
In this section, we visualize and analyze the test trials of each evaluated policy as discussed in Section~\ref{sect:explore_in_unseen} on the \texttt{PointMaze} series tasks. Specifically, we examine the trials of the Max Occ for each policy, which are depicted in Fig.~\ref{fig:point_test_vis}.

As outlined in Section~\ref{sect:explore_in_unseen}, the evaluated policies are categorized into two types: GC policies and skill policies. Due to the similar performance across GC policies, we use the same visualizations for all, which are collectively presented alongside the learned GC policy. Consequently, our analysis includes five rows representing different policies: the learned GC policy, uniform GC policy, GEASD-L skill policy, GEASD-H skill policy, and uniform skill policy. Each category is depicted with three figures corresponding to the tasks \texttt{Spiral}, \texttt{Spiral-C}, and \texttt{Serpentine}. These figures feature trajectories plotted in blue and decision points for actions or skill decisions marked with red dots.

We begin by analyzing the behaviors that arise from learned GC policies. In the lower-left grid of \texttt{Spiral-C}, these policies make rightward actions similar to those in the training environment, \texttt{Spiral}, resulting in the agent becoming stuck against the wall in \texttt{Spiral-C}. Likewise, in the rightmost grid of the second row in \texttt{Serpentine}, the upward movements executed by the learned GC policies—again mirroring those in \texttt{Spiral}—are obstructed by the wall. Consequently, the behaviors of the learned GC policies have not adapted to structural changes such as those involving walls; instead, they primarily respond to coordinate information consistent with these two tasks. This bias restricts the transferability of knowledge from the training environments to other scenarios with similar structural elements. In contrast, the actions under uniform policies are noisy, yet their impartial decisions facilitate broader exploration in both new environments compared to the learned GC policies.

Regarding skill policies, as detailed in Section~\ref{sect:explore_in_unseen}, they exhibit markedly improved exploration progress compared to GC policies. As illustrated in Fig.~\ref{fig:point_test_vis}, the best trial of both GEASD-H and GEASD-L skill policies on \texttt{Spiral-C}, achieving 100\% Max Occ, thoroughly reaches the center of the clockwise spiral maze. In \texttt{Serpentine}, the best trial with a 70\% occupancy ratio of the GEASD-H skill policy explores seven out of ten rows, significantly surpassing the performance of the GC policies. Furthermore, the best trial of the GEASD-L policy (Max Occ 96\%) concludes merely two grids short of the target grid, closely approaching success. For comparison, the best trial under the uniform skill policy (Max Occ 36\%) only explores three and a half rows, half as many as that achieved by GEASD-H. This comparison underscores that the benefits of GEASD skill policies over GC policies extend beyond merely introducing new skills. 

In summary, unlike GC policies, which struggle to adapt to new environments, GEASD skill policies successfully navigate these challenges by enhancing local entropy via effectively leveraging structural information. This enables them to better overcome barriers to application in unfamiliar tasks.

\begin{table}[ht]
    \centering
    \begin{tabular}{m{0.13\linewidth}|m{0.15\linewidth}m{0.15\linewidth}m{0.15\linewidth}}
        \hline
        \multicolumn{1}{c|}{Method} & \multicolumn{1}{c}{\texttt{Spiral}} & \multicolumn{1}{c}{\texttt{Spiral-C}} & \multicolumn{1}{c}{\texttt{Serpentine}} \\
        \hline
        Learned GC Policy &\includegraphics[width=\linewidth]{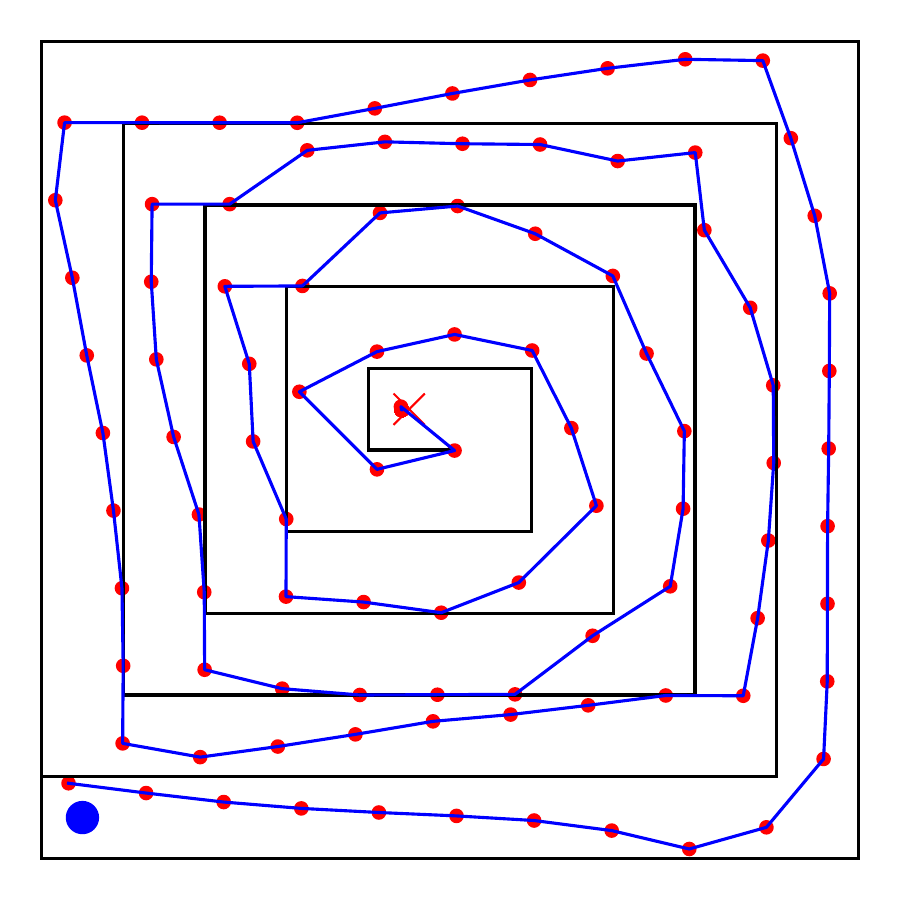} & \includegraphics[width=\linewidth]{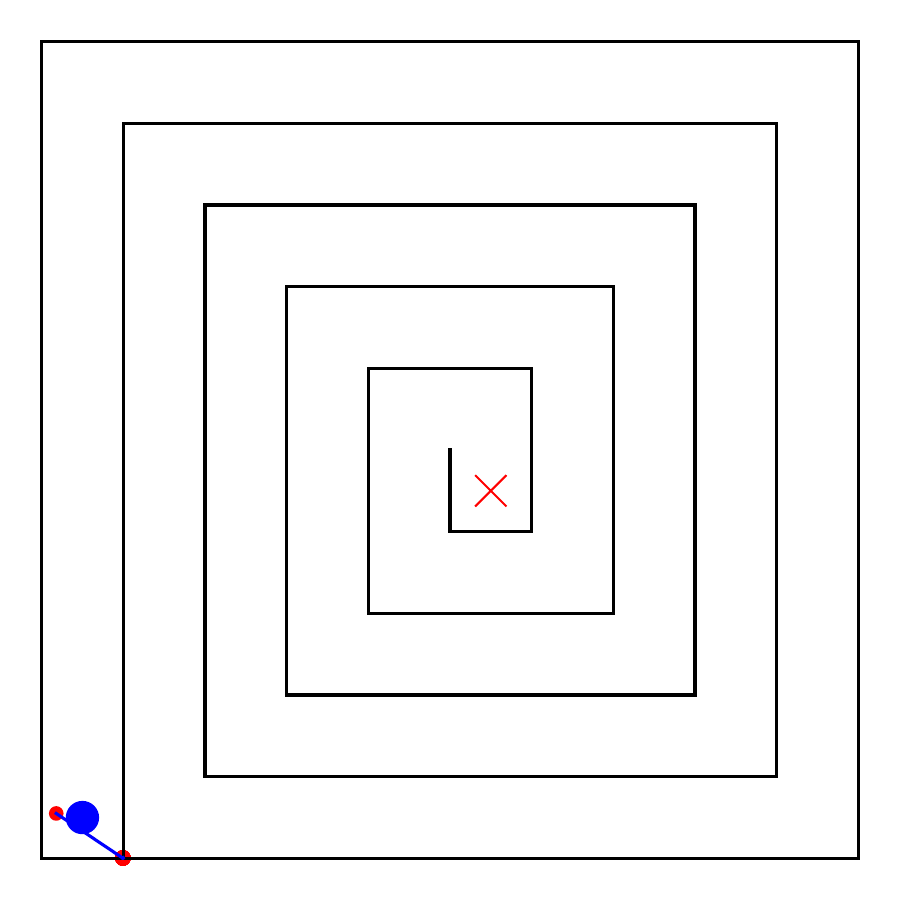} & \includegraphics[width=\linewidth]{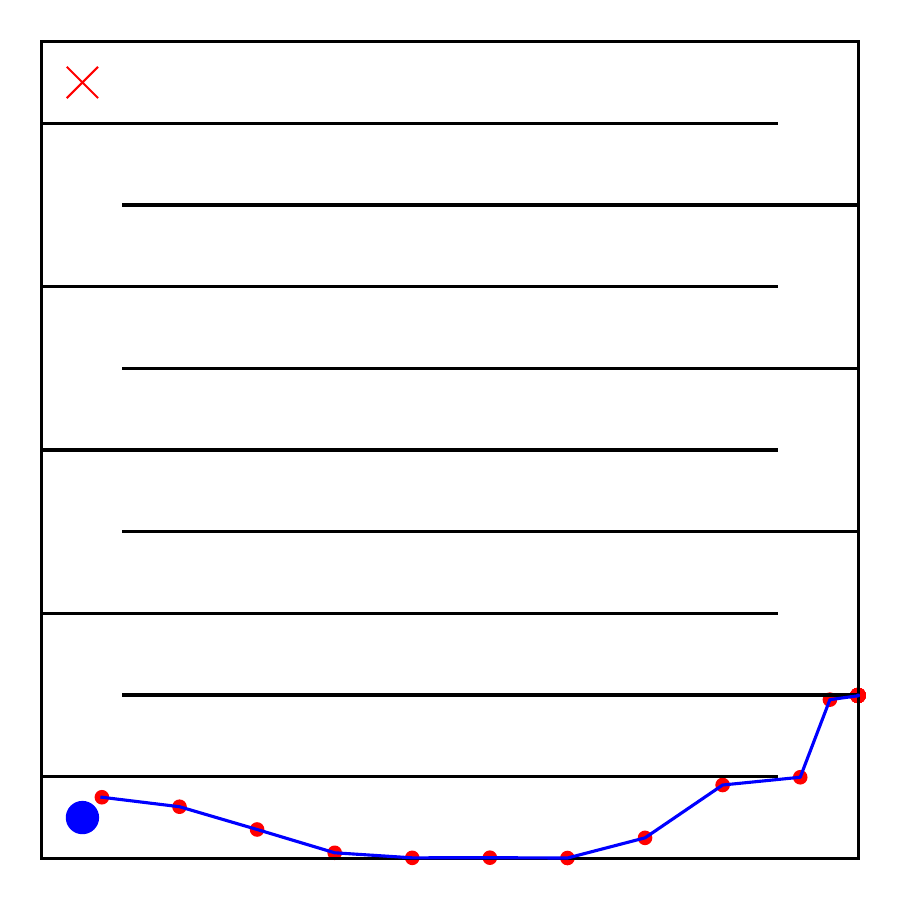} \\
        Uniform GC Policy &\includegraphics[width=\linewidth]{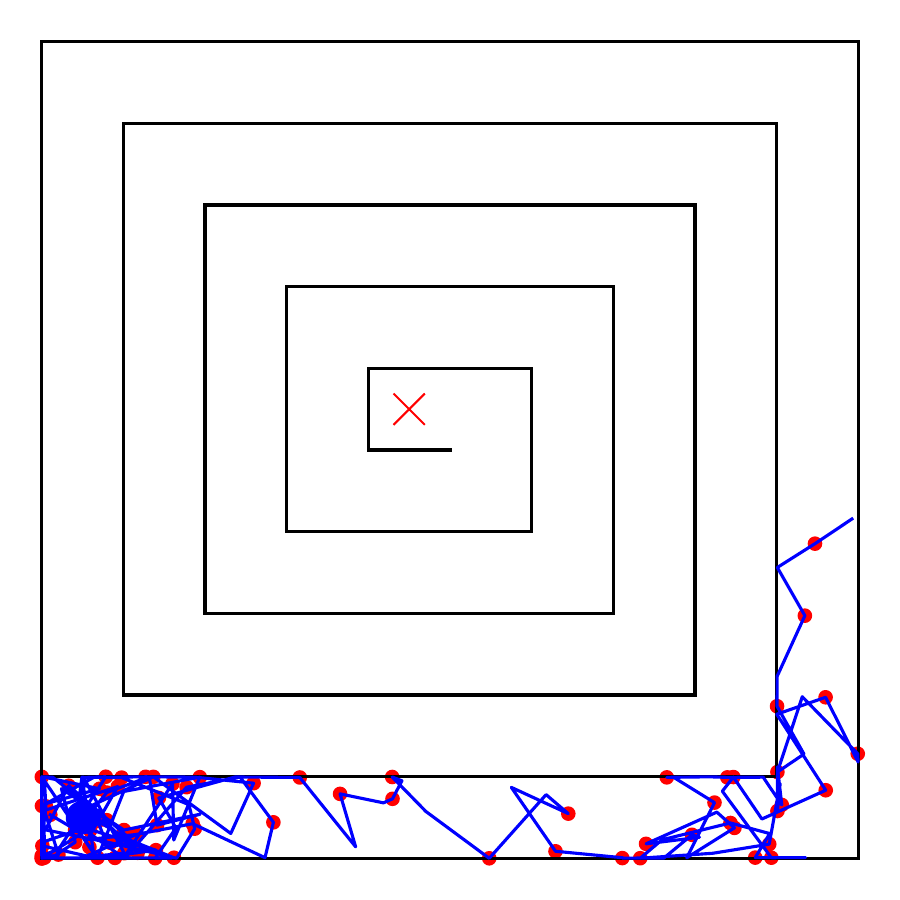} & \includegraphics[width=\linewidth]{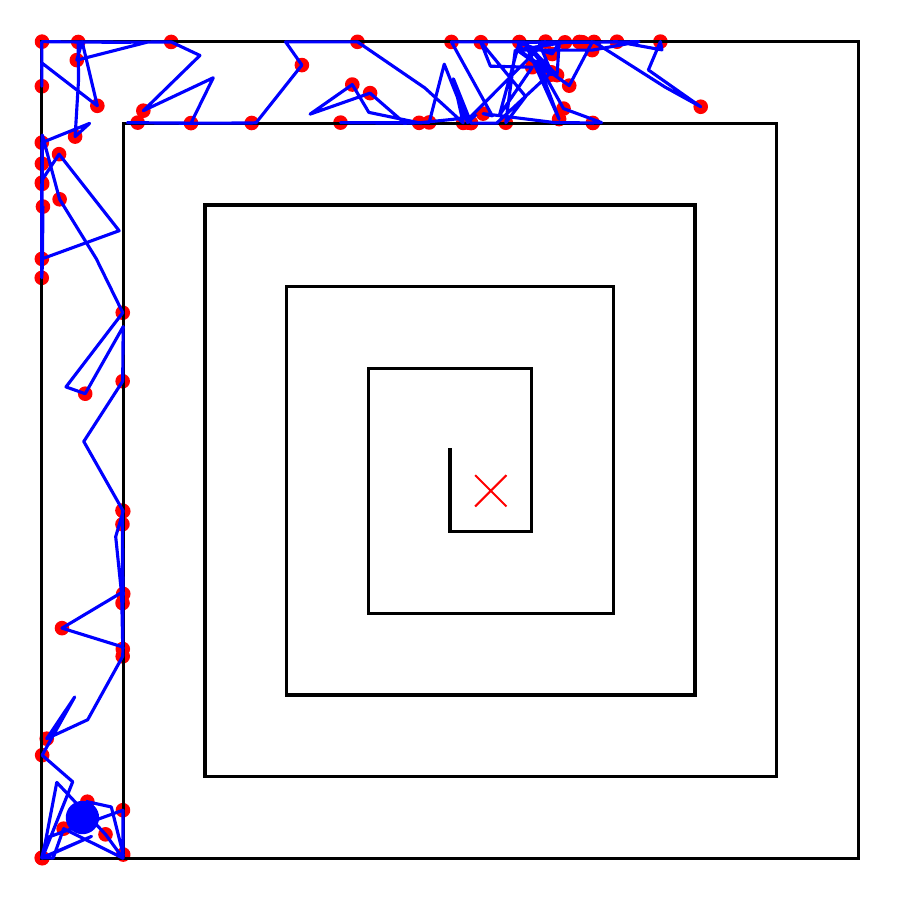} & \includegraphics[width=\linewidth]{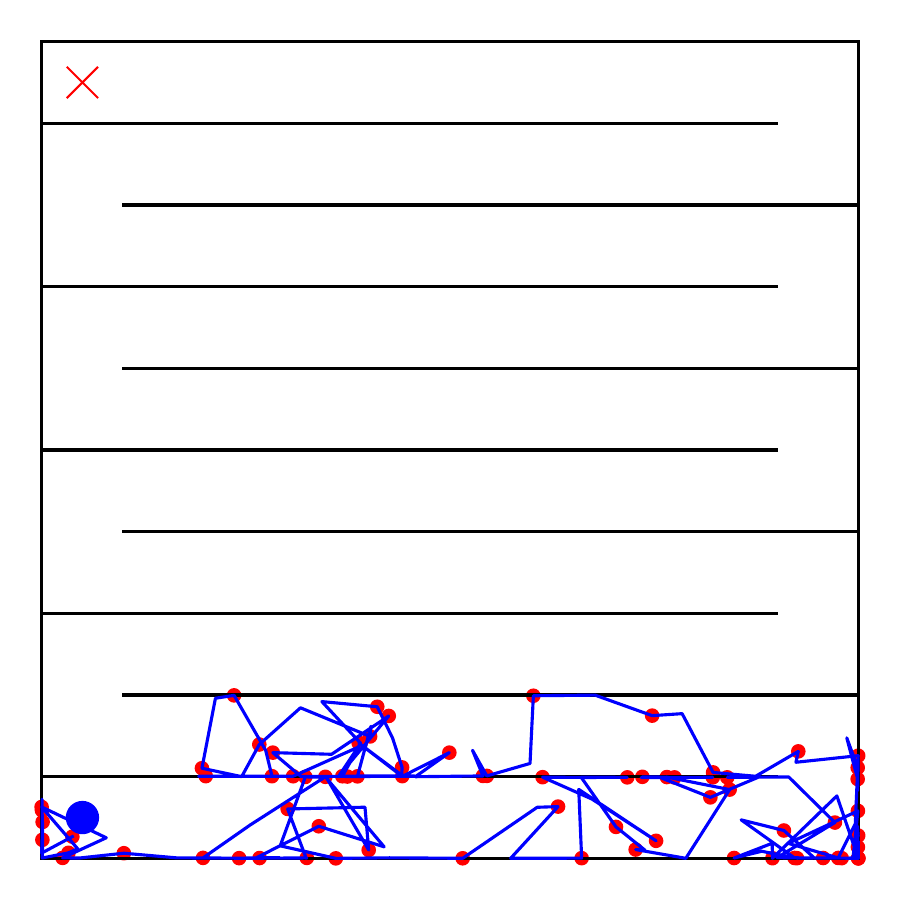} \\
        GEASD-H Skill Policy &\includegraphics[width=\linewidth]{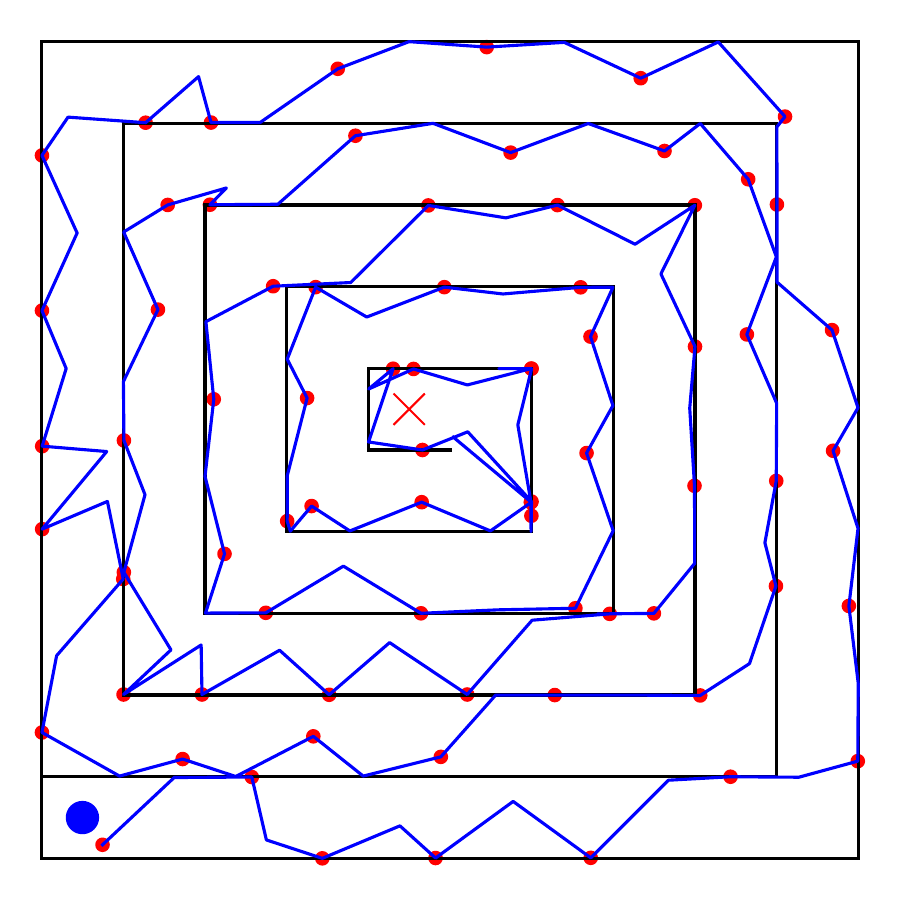} & \includegraphics[width=\linewidth]{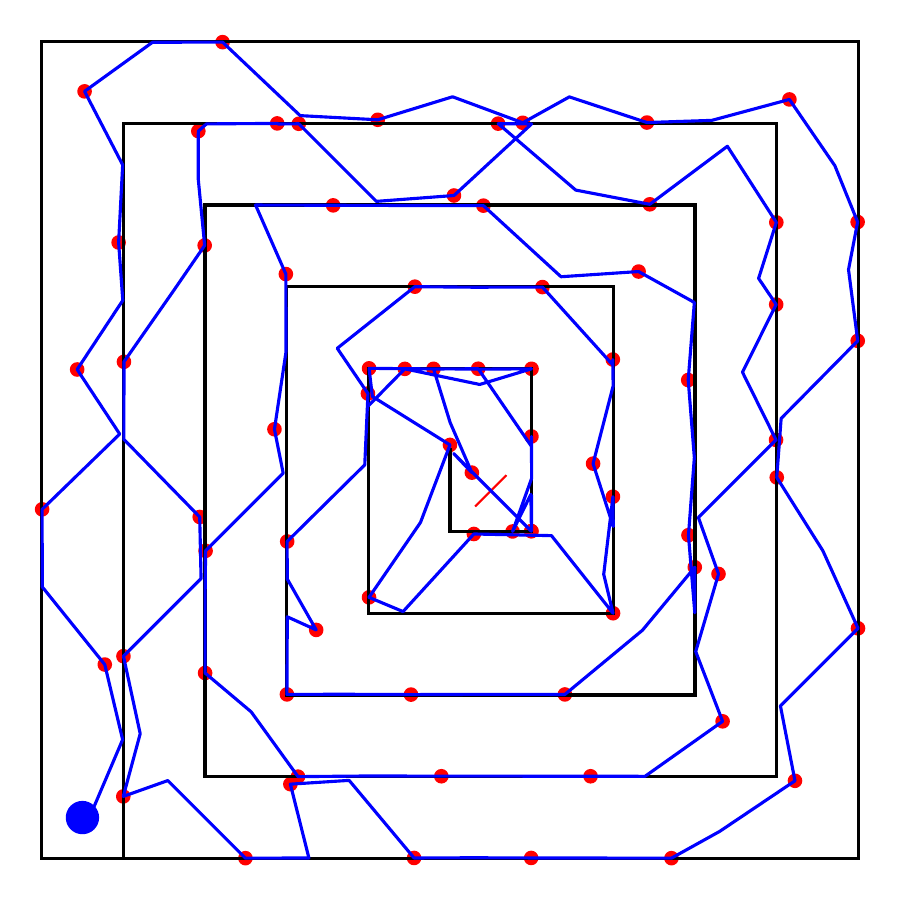} & \includegraphics[width=\linewidth]{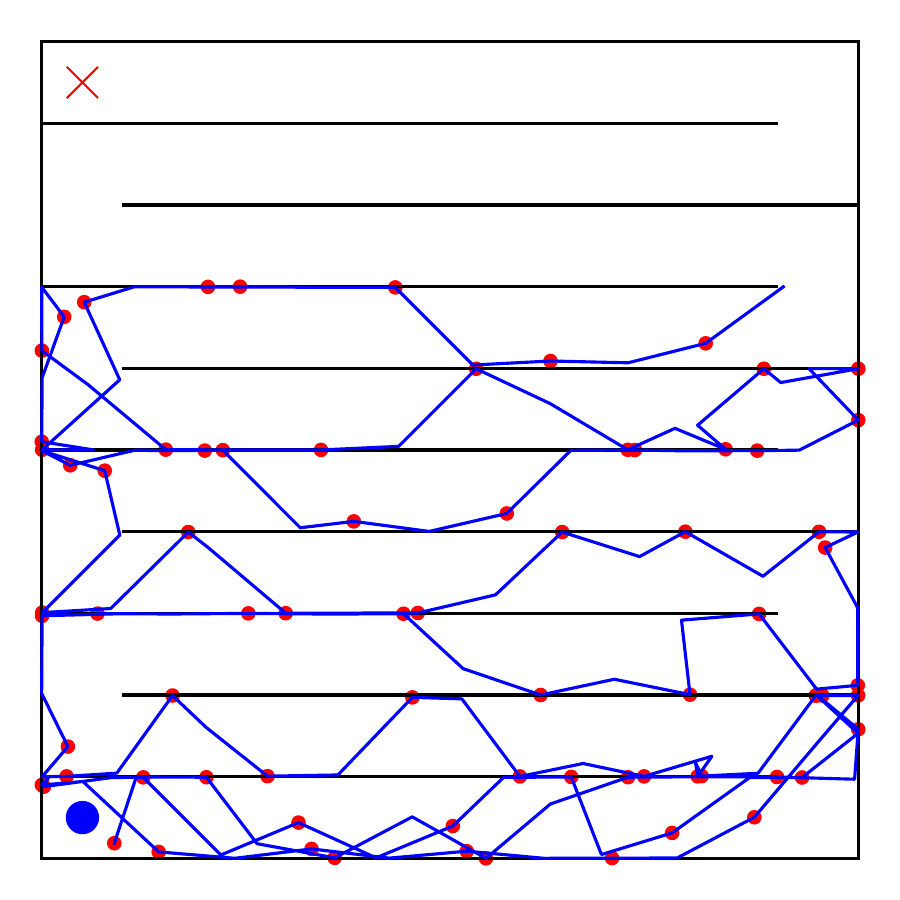} \\
        GEASD-L Skill Policy & \includegraphics[width=\linewidth]{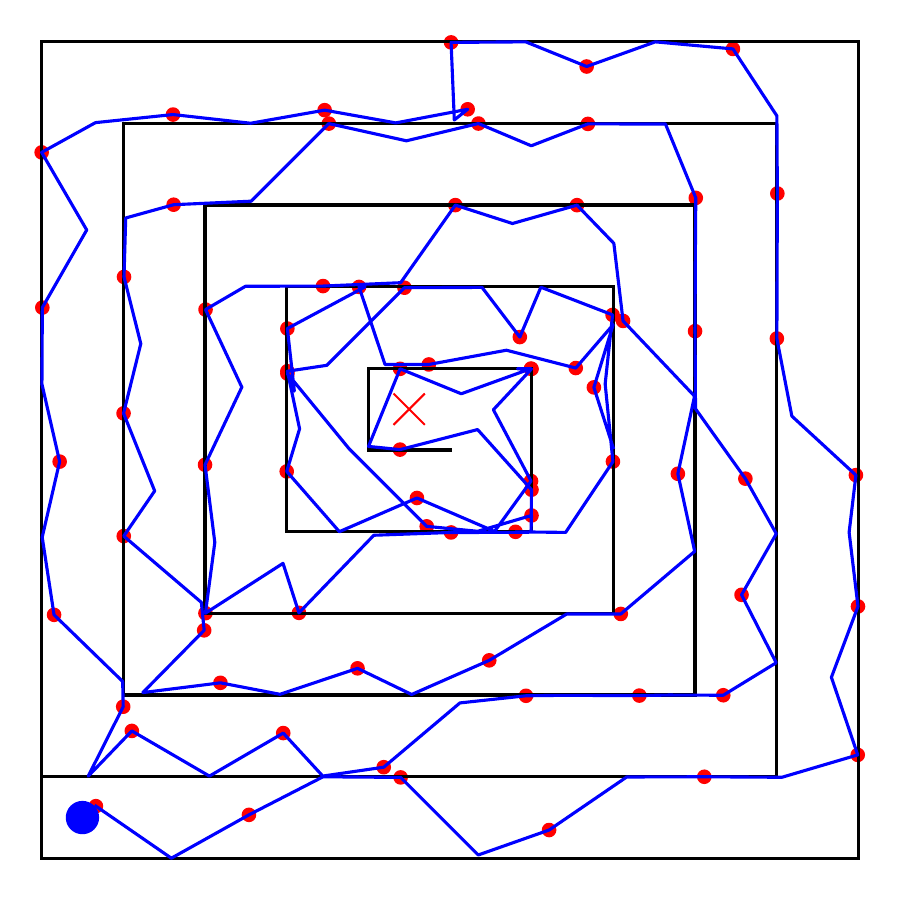} & \includegraphics[width=\linewidth]{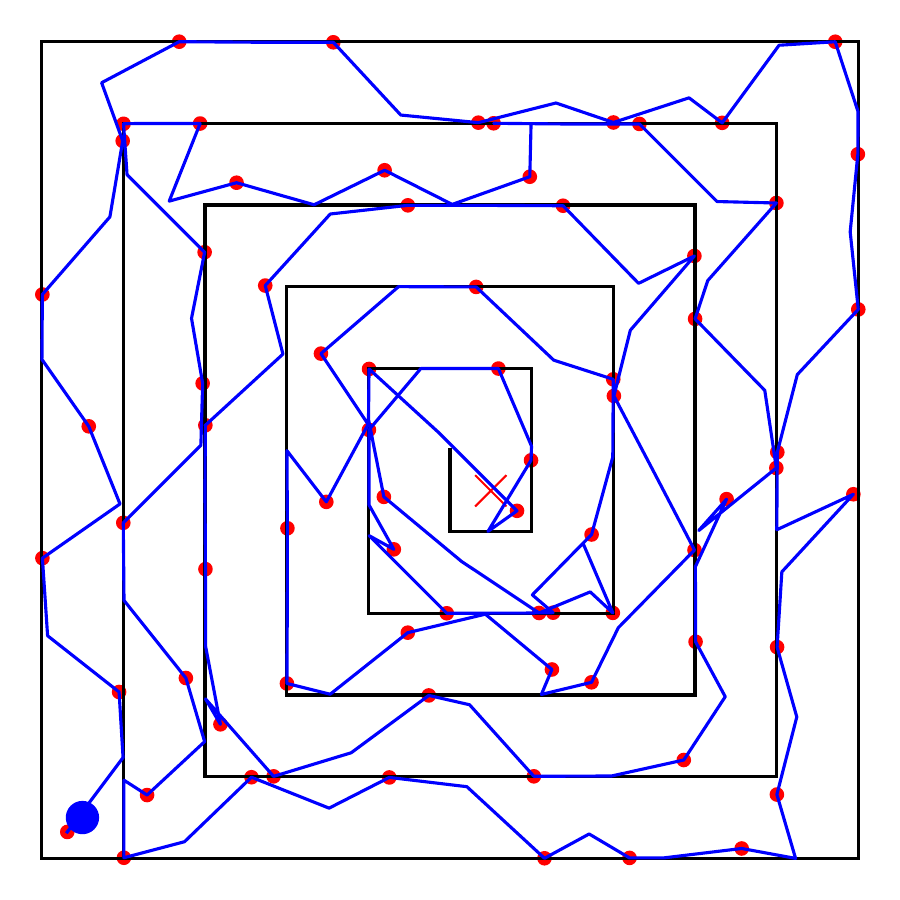} & \includegraphics[width=\linewidth]{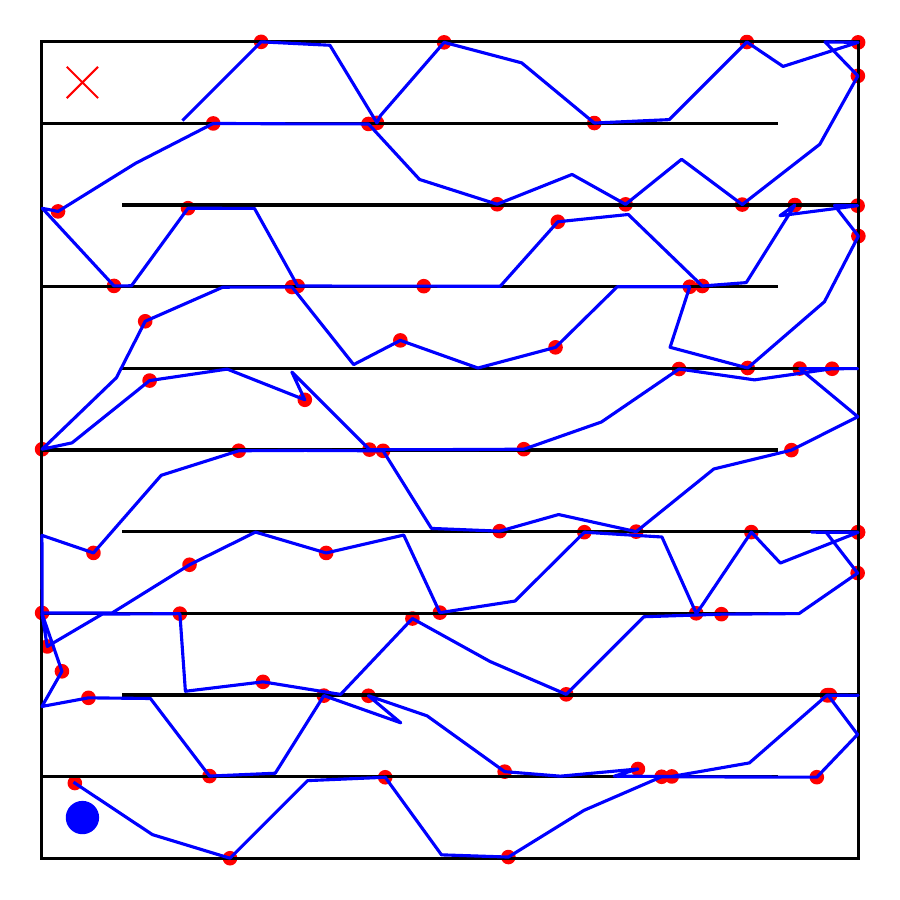} \\
        Uniform Skill Policy & \includegraphics[width=\linewidth]{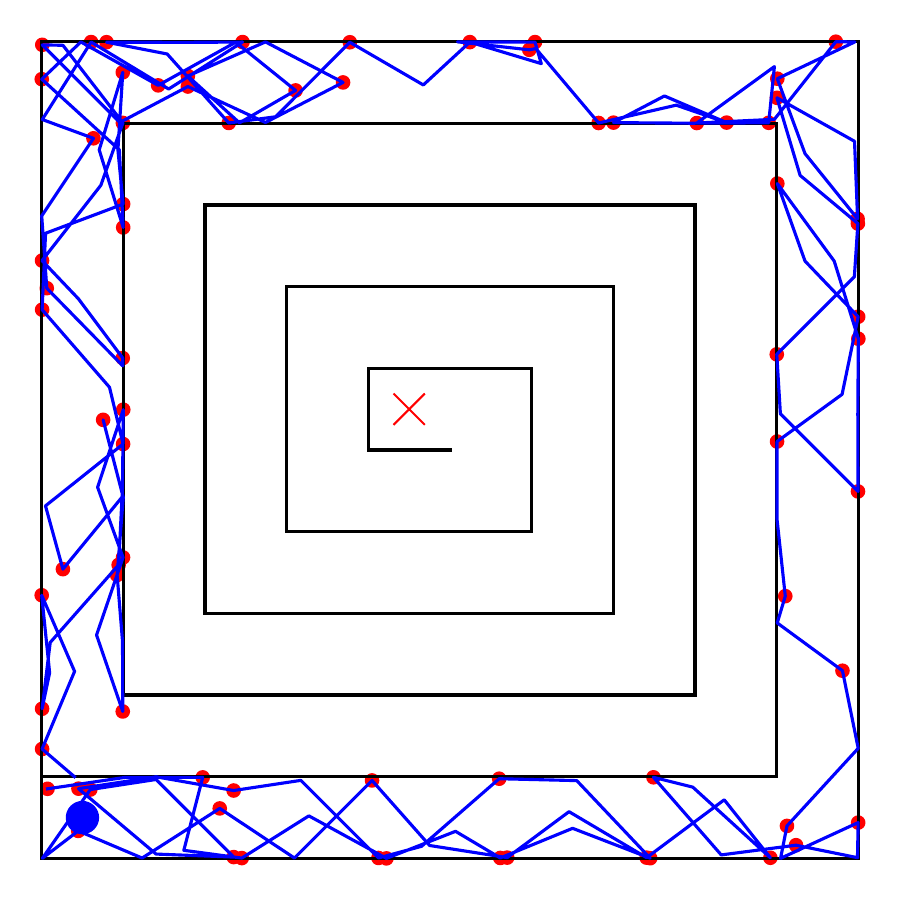} & \includegraphics[width=\linewidth]{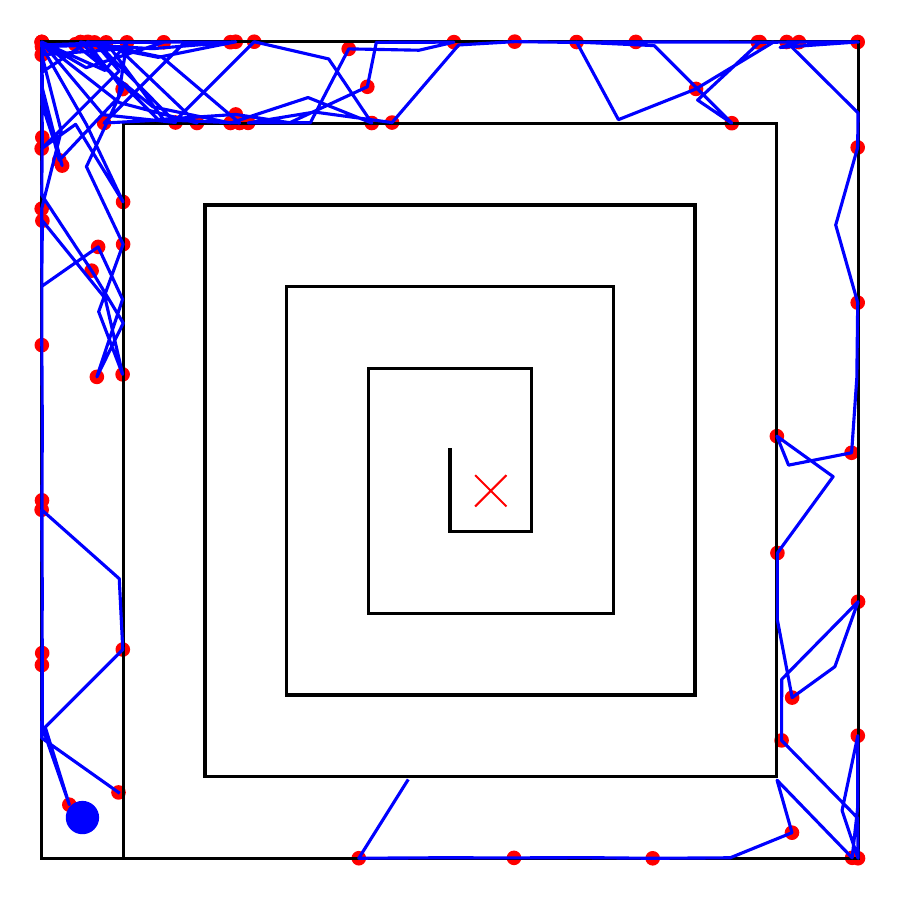} & \includegraphics[width=\linewidth]{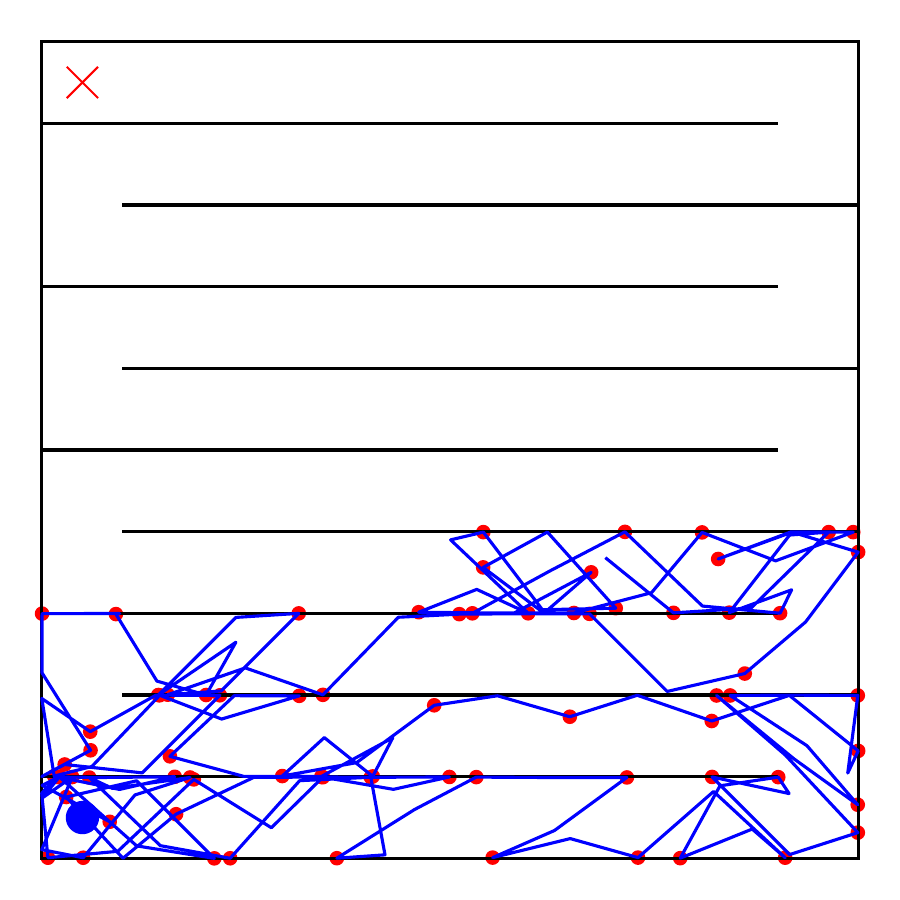} \\
    \end{tabular}
    \caption{Visualized trajectories of selected policies with maximum occupancy ratios (Max Occ) in the \texttt{PointMaze} task series over 50 evaluation episodes.}
    \label{fig:point_test_vis}
\end{table}

\section{Additional Ablation Study}
In this section, we conduct two supplementary ablation studies. The first study assesses the benefits of the GEASD-L approach compared to GEASD-H, specifically highlighting the advantages of utilizing all available data in the replay buffer for learning skill value functions, as opposed to exclusively using skill execution data.
The second study investigates the impact of varying the sizes of the contextual horizon \(C\) on the learning efficiency, while maintaining a consistent skill horizon \(k=2\) on the \texttt{PointMaze-Spiral} environment.

\subsection{Learning Data for Skill Value Functions}
As discussed in Section~\ref{sect:structural_representation}, GEASD-H and GEASD-L exhibit differences not only in their methods of calculating target values but also in the range of training data utilized for skill value functions. Specifically, GEASD-L is capable of leveraging data collected by goal-conditioned policies via importance sampling, in addition to skill execution data.

This section investigates the benefits to GEASD-L when it incorporates additional non-skill execution data in learning the skill value functions. This issue is closely tied to whether the extra data can help the agent learn accurate skill value functions more swiftly. To facilitate this investigation, we introduce a variant of GEASD-L that solely relies on skill execution data for learning skill value functions. To clearly differentiate between the variants trained on varying data scopes, we label each method according to the data used for skill value functions. The original GEASD-L, as referenced in the main text, is denoted as GEASD-L(all data), while the new variant focused on skill execution data is labeled as GEASD-L(skill data).

As shown in Fig.~\ref{fig:ablation_study_data}, GEASD-L(skill data) degrades significantly compared to GEASD-L(all data). In \texttt{PointMaze-Spiral}, GEASD-L(skill data) reaches 100\% success 17\% and 33\% slower than GEASD-H(skill data) and GEASD-L(all data). In \texttt{AntMaze-U}, the learning trend of GEASD-L(skill data) is similar to GEASD-H(all data), with its average success being capped at around 60\% at the end of training. In contrast, the GEASD-L(all data) can reach 100\% at the end of training. Therefore, it demonstrates the capability of GEASD-L to leverage all data via importance sampling actually improve the overall learning efficiency, which can result from the facilitation in learning the skill value functions with more data sources. It is alos woth noting the performance of GEASD-L(skill data) is even worse than GEASD-H(skill data) in \texttt{PointMaze-Spiral}. The reason can be that the GEASD-L suffers from inaccurate value estimations as the exponential discount factor beyond $k$ steps is still $0.25$ in our cases, bootstapping the future inaccurate skill value fucntions in the  target values. In contrast, GEASD-H only calculates the accumulation of rewards over the next $k$ step as the target values, which avoids the inherent inaccurate evaluations.

\begin{figure}[ht]
    \centering
    \includegraphics[width=0.9\textwidth]{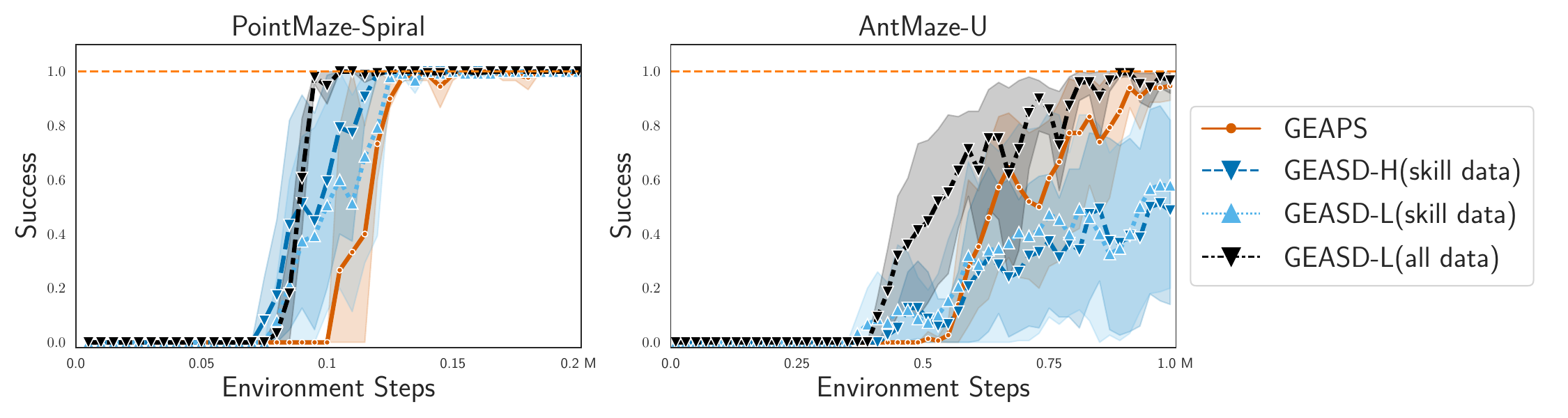}
    \caption{A comparative ablation study on achieving desired goal distribution: analyzing the impact of training data of skill value functions throughout training.}
    \label{fig:ablation_study_data}
\end{figure}

\subsection{The Contextual Horizon}
In this section, we explore the influence of the contextual horizon on learning efficiency. The contextual horizon defines the extent of historical information the agent considers and the scope of local entropy it aims to optimize. We specifically assess the performance of our most effective GEASD-L variant at different contextual horizons—$3$, $5$, and $10$—within the \texttt{PointMaze-Spiral} environment. Each variant is distinguished by labeling it with its contextual horizon enclosed in brackets; for instance, the configuration used in the main text for GEASD-L on \texttt{PointMaze-Spiral} is labeled as GEASD-L($C=10$).

As depicted in Fig.~\ref{fig:ablation_study_contextual_horizon}, GEASD-L($C=10$) outperforms the other variants. We observe a decline in learning efficiency as the contextual horizon decreases from $10$ to $3$. Notably, GEASD-L($C=10$) achieves 100\% success across random seeds 12\% and 25\% faster than the variants with $C=5$ and $C=3$, respectively. Particularly, at a contextual horizon of $C=3$, which is quite close to the skill horizon $k=2$, the learning trend is similar to that observed with GEAPS. This similarity can be explained by the optimization objective of GEASD-L $H(\bm{\Phi}(h^C_{t+k})|h^C_t)$ converging towards the optimization objective of GEAPS $H(\bm{\Phi}(h^C_{t+k}))$ as $C$ nears $k$.

\begin{figure}[h]
    \centering
    \includegraphics[width=0.56\textwidth]{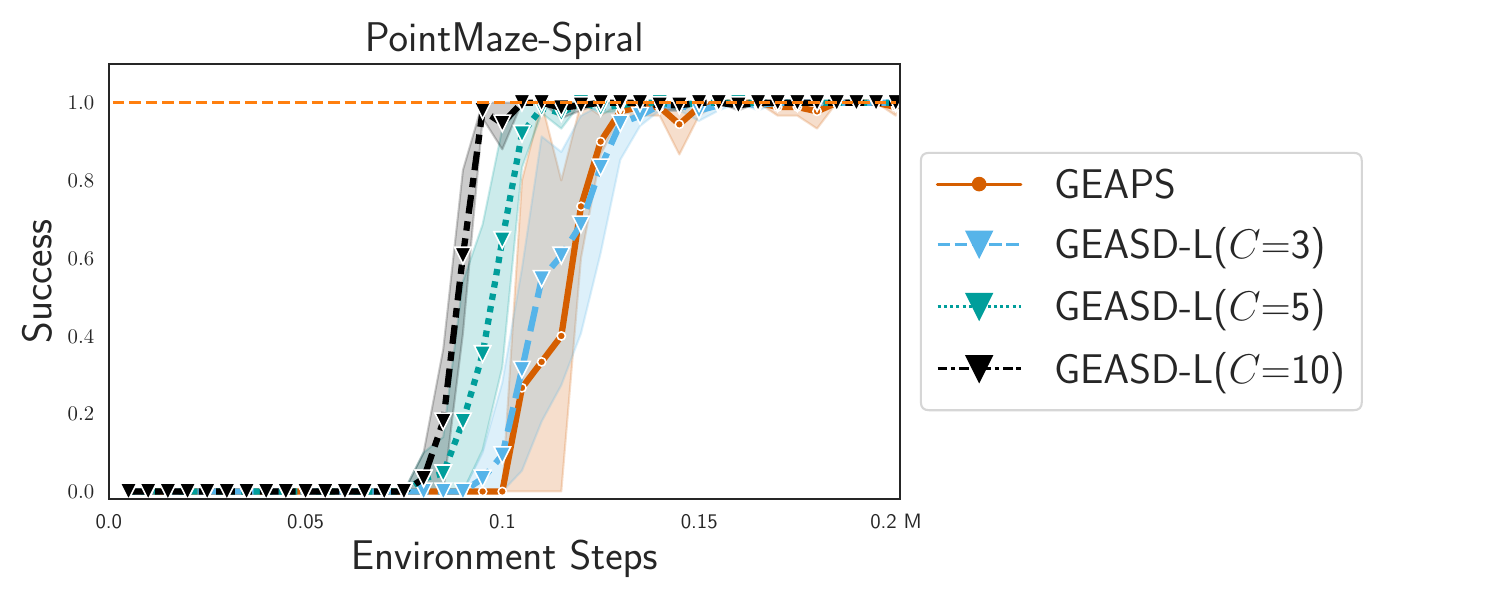}
    \caption{Varying Contextual Horizons ($C$)}
    \label{fig:ablation_study_contextual_horizon}
\end{figure}

\bibliography{ref}
\bibliographystyle{plainnat}

\end{document}